\documentclass{article}
\usepackage{geometry}

\usepackage{microtype}
\usepackage{graphicx}
\usepackage{subfigure}
\usepackage{booktabs} 
\usepackage{amsmath, amsthm, amsfonts, amssymb}
\usepackage{pifont}
\usepackage{algorithm, algorithmic}
\usepackage{appendix}
\usepackage[colorlinks, citecolor=blue]{hyperref}
\usepackage{color}
\usepackage{url}

\newcommand{\mS}{\mathcal S}
\newcommand{\mA}{\mathcal A}

\newcommand{\gap}{\mathrm{gap}}
\newcommand{\clip}{\mathrm{clip}}
\newcommand{\conc}{\mathrm{conc}}

\newtheorem{Theorem}{Theorem}
\newtheorem{Lemma}{Lemma}
\newtheorem{Assumption}{Assumption}
\newtheorem{Proposition}{Proposition}

\title{Gap-Dependent Bounds for Two-Player Markov Games}
\author{
Zehao Dou\\Yale University\\\texttt{zehao.dou@yale.edu}\\
\and Zhuoran Yang\\ Princeton University\\ \texttt{zy6@princeton.edu}\\
\and Zhaoran Wang\\ Northwestern University\\ \texttt{zhaoranwang@gmail.com}\\
\and Simon S.Du \\ University of Washington\\\texttt{ssdu@cs.washington.edu}
}
\date{}

\begin{document}
\maketitle
\allowdisplaybreaks[4]

\begin{abstract}

As one of the most popular methods in the field of reinforcement learning, Q-learning has received increasing attention. Recently, there have been more theoretical works on the regret bound of algorithms that belong to the Q-learning class in different settings. In this paper, we analyze the cumulative regret when conducting Nash Q-learning algorithm on 2-player turn-based stochastic Markov games (2-TBSG), and propose the very first gap dependent logarithmic upper bounds in the episodic tabular setting. This bound matches the theoretical lower bound only up to a logarithmic term. Furthermore, we extend the conclusion to the discounted game setting with infinite horizon and propose a similar gap dependent logarithmic regret bound. Also, under the linear MDP assumption, we obtain another logarithmic regret for 2-TBSG, in both centralized and independent settings.     
\end{abstract}
\newpage
\tableofcontents
\newpage
\section{Introduction}
Recently, designing an effective and efficient algorithm to obtain a near-optimal strategy in sequential decision-making tasks has attracted more and more interest in the field of modern reinforcement learning (RL) \cite{sutton1988reinforcement}. By estimating the optimal state-action value function (a.k.a Q-function), Q-learning method \cite{watkins1992technical} is one of the most popular classes of algorithms. In each iteration of Q-learning algorithms, the agent greedily chooses the action with the largest Q value and achieve his reward and the next state by interacting with the underlying RL environment. At the same time, the algorithm keeps updating the Q-values by using Bellman Equation. In comparison, the model-based methods attempt to reveal the structure of the environment, which lead to more memory and worse time efficiency. That's why Q-learning, a typical model-free method, is playing an important role in a wide range of RL problems \cite{mnih2013playing, mnih2015human}.

In this paper, we study a more complicated scenario, two-player turn-based stochastic Markov game, which is a special case of multi-agent reinforcement learning (MARL). In this setting, two players known as the max-player and the min-player interact with each other one by one and optimize their individual rewards. The max-player's goal is to maximize the cumulative reward while the min-player attempts to minimize it. In order to measure the quality of the two players' policies, we study the \textbf{total regret}, where the players learn a policy tuple $(\pi_k, \mu_k)$ for a sequence of episodes $k=1,2,\ldots,K$, and suffer a total regret, which is the total sub-optimality of the policies $\pi_1,\pi_2,\ldots,\pi_K$. From the regret minimization perspective, many related works \cite{kearns1998finite, sm2003on, azar2013minimax, jin2018is, dong2019q, liu2020regret, bai2020near} have provided a $\sqrt{K}$-type upper bound of total regret in various settings where $K$ is the number of episodes. Although these $\sqrt{K}$-type upper bounds are easy to understand and match the lower bound, they paint an overly pessimistic worst-case scenario on the Markov decision processes. 

Under this circumstance, recent works \cite{yang2020q,ok2018exploration,he2020logarithmic, simchowitz2019non,xu2021fine} limit the MDP with certain structures and propose much tighter upper bound or provide new perspectives. One line of works assume the existence of minimal sub-optimality gap, and establish the $C\log K$-type of total regret upper bound, where $C$ is an instance-dependent constant associated with the minimal sub-optimality gap, named $\gap_{\min}$:
\[\gap_{\min}:=\min\{\gap(s,a) = V^*(s)-Q^*(s,a)>0\}.\]
Here, $V^*$ and $Q^*$ denote the value function and Q-function for an optimal policy $\pi^*$. 

Another line of works \cite{xie2020learning, jin2019provably, he2020logarithmic} assume a certain structure of reward function and probability transition kernel, such as the linear function approximation. For instance, \cite{jin2019provably} studies the episodic MDPs with linear MDP assumptions, which means that both the transition probability function and reward function can be represented as a linear function of a given feature mapping. \cite{he2020logarithmic} combines the two ideas and provides a gap-dependent logarithmic regret bound with linear function approximation. 

However, all the provable logarithmic regret bounds are under the single agent setting where the sole player attempts to achieve the highest cumulative rewards by interacting with the underlying environment. For the two-player settings or multi-agent settings, the gap-dependent logarithmic regret bound still remains absent because the sub-optimality gaps are not non-negative and the concept of minimal sub-optimality gap no longer makes sense in these settings. In this paper, we overcome these difficulties and propose a gap-dependent logarithmic cumulative regret bound of 2-TBSG in all the tabular, discounted and linear function expression settings for the first time. Next, we introduce the three main contributions of this paper, which are listed below.

\begin{itemize}
\item In two-player episodic turn-based general sum stochastic games (2-TBSG) with finite horizon \cite{jia2019feature, shapley1953stochastic}, we propose a new concept named minimal positive sub-optimality gap $\gap^+_{\min}$. Based on that, we provide a gap-dependent logarithmic total regret bound $\mathcal O\left(\frac{H^6SA\log(SAT)}{\gap^+_{\min}}\right)$ when using the Nash Q-learning algorithm.

\item Based on the result above, we further extend it to the discounted 2-TBSG with infinite horizon, and obtain a total regret upper bound $\mathcal O\left(\frac{SA}{\gap^+_{\min}(1-\gamma)^5\log(1/\gamma)}\cdot\log\frac{SAT}{\gap^+_{\min}(1-\gamma)}\right),$ which is gap dependent and logarithmic on $T$. 

\item We analyze the finite-horizon 2-TBSG with linear function expression. Under the linear MDP assumption, we propose the 2-TBSG version of LSVI-UCB algorithm \cite{jin2019provably} and provide a provable 
$\tilde{\mathcal O}\left(\frac{d^3H^5\log(16dK^3H^3)}{\gap_{\min}^+}\right)$
total regret upper bound which is also gap dependent and logarithmic on $K$, in both centralized and independent settings.

\end{itemize}

Technically, we are using a new set of algorithms when solving the pure Nash Equilibrium of 2-TBSG, which provably exists, and it makes our proof novel and more difficult. 
Compared with \cite{xie2020learning} that establishes $\sqrt{T}$-regret, to achieve the gap-dependent  logarithmic regret, we utilize a different regret decomposition method links the regret to a sum of gap terms, which are further bounded via a peeling argument. See Sections \ref{sec:split_regret} and \ref{sec:peeling} for details. Meanwhile, compared with \cite{he2020logarithmic, du2019provably}, we face new difficulties since the game setting we are analyzing has
a max-player as well as a min-player. Therefore, we need to propose a new technique to control the influence of the opponent, which results in developing both upper confidence bound (UCB) and lower confidence bound (LCB) at the same time. 
To our best knowledge, our result establish the first logarithmic regret bounds for zero-sum Markov games under both the tabular and linear   settings.  
\section{Related Works}
\paragraph{Tabular and Infinite-horizon MDP}  There is a long list of results focusing on the regret or sample complexity on tabular episodic MDPs and discounted MDP with infinite horizon. They can be recognized as model-free methods or model-based methods, which are two different types of methodology. Model-based methods \cite{jaksch2010near, dann2017unifying, osband2016generalization} explicitly estimate the transition probability function while the model-free methods \cite{jin2018is, strehl2006pac} do not. One line of works \cite{sidford2018variance, lattimore2012pac, ghavamzadeh2011speedy, wainwright2019variance, azar2012on, koenig1993complexity} assume the existence of a simulator (also called a generative model) where the agent can freely query any state-action pair to the underlying environment and return the reward as well as the next state. In the episodic setting without simulators, \cite{jin2018is} achieves a $\tilde{\mathcal O}(\sqrt{H^3SAT})$ regret bound for a model-free algorithm and \cite{azar2017minimax} proposes a UCB-VI algorithm with Bernstein style bonus with achieves a $\tilde{\mathcal O}(\sqrt{H^2SAT})$ regret bound for a model-based algorithm. Both of the two upper bounds nearly attain the minimax lower bound $\Omega(\sqrt{H^2SAT})$ \cite{jaksch2010near,jin2018is, osband2016on}. Recently, \cite{zhang2020almost, jin2018is} provide a $\sqrt{T}$-type regret bound for Q-learning algorithms (which is a widely used model-free algorithm) where $T$ is the number of episodes. 

Another line of works focus on providing $\log T$-type regret bound based on instance-dependent quantities. \cite{ok2018exploration} shows us that $\log T$ is unavoidable as a lower bound. \cite{tewari2007optimistic} proposes an OLP algorithm for average-reward MDP and achieves an asymptotic logarithmic regret $\mathcal O(C\log T)$ where the constant $C$ is instance-dependent. \cite{jaksch2010near} provides a UCRL2 algorithm and provides a non-asymptotic regret bound $\mathcal O(D^2S^2A\log T/\gap_{\min})$ where $D$ is the diameter and $\gap_{\min}$ is the minimal sub-optimality gap. A recent work \cite{wang2019optimism} proves a $\mathcal O(SAH^6\log(SAT)/\gap_{min})$ regret bound for the model-free optimistic Q-learning algorithm. 

\paragraph{Linear Function Approximation} A recent line of works \cite{jin2018is, wang2019optimism, yang2020reinforcement, jia2019feature, zanette2020provably, du2019provably, du2020agnostic, zhou2020provably, weisz2020exponential} solve MDP with linear function approximations and propose a regret bound. After parameterizing the Q-function with feature mapping under the linear MDP assumption, \cite{jin2019provably} proposes LSVI-UCB algorithm with $\tilde{\mathcal O}(\sqrt{d^3H^3T})$ total regret bound. \cite{zanette2020provably} improves the bound to $\tilde{\mathcal O}(\sqrt{d^2H^2T})$ by introducing a global planning oracle. Later, a new linear mixed model assumption is proposed and a number of works \cite{jia2019feature, jia2020model} introduce the UCLR-VTR algorithm to solve MDP under the new assumption, and provide $\tilde{\mathcal O}(\sqrt{d^2H^3T})$ regret bound. In the discounted setting, \cite{zhou2020provably} provides a $\tilde{\mathcal O}(d\sqrt{T}/(1-\gamma)^2)$ upper bound where $\gamma$ is the discount ratio.

\section{Preliminaries and Notations}
In this section, we introduce some important concepts, notations and background knowledge.
\paragraph{Setting of Two-player Turn-based Stochastic Games} In two-player turn-based games (2-TBSG), only one player takes his action at each step. Denote the two players as $P_1, P_2$, which are the max-player and the min-player respectively. We partition the whole action space as $\mA=\mA_1\cup\mA_2$, where $\mA_i$ is the state space of player $P_i$. Since the stochastic game is episodic under our setting, we denote $2H$ as the number of steps in one episode. At each step $h\in[2H]$, when $h$ is an odd number, it's the max-player $P_1$'s turn to observe the current state $s_h$ and take an action $a_h\in\mA_1$, and then we receive the reward $r_h(s_h,a_h)$. Similarly, when $h$ is an even number, the min-player $P_2$ observes the current state $s_h$ and takes the action $a_h\in\mA_2$, and then they receive the reward $r_h(s_h,a_h)$. After taking the action, the system makes transition to a new state $s_{h+1}\sim \mathbb{P}_h(\cdot|s_h, a_h)$. For each action $a\in\mA$, denote $I(a)\in\{1,2\}$ as it indicates of the current player to play, so that $a\in \mA_{I(a)}$.

\paragraph{General Notations} We denote the tabular episodic Markov Game as $\mathrm{MG}(2H,\mS,\mA_1,\mA_2,\mathbb{P},r)$, where $2H$ is the number of steps in each episode, $\mS$ is the set of states, and $(\mA_1,\mA_2)$ are the action sets of the max-player and the min-player respectively. Since we analyze the two-player turned base games in this paper, the odd number steps are the max-player's turn and the even number steps are the min-player's turn. $\mathbb{P}=\{\mathbb{P}_h\}_{h\in[2H]}$ is the collection of transition matrices and $\mathbb{P}_{h}(\cdot|s,a)$ outputs the probability distribution over states when $a$ is the action taken at step $h$ after state $s$. $r_h: \mS\times\mA\rightarrow [0,1]$ is a deterministic reward function at step $h$.

\paragraph{Markov policy, Q-function and Value function} We denote $(\pi,\mu)$ as the Markov policy of the max-player and the min-player respectively. Max-player's policy $\pi$ is a collection of $H$ mappings $\{\pi_h:\mS\rightarrow \mA_1\}_{h=1,3,\cdots,2H-1}$, which maps the current state to the action. Similarly, min-player's policy $\mu$ is a collection of $H$ mappings $\{\mu_h:\mS\rightarrow \mA_2\}_{h=2,4,\cdots,2H}$. We denote $\pi_h(s)$ and $\mu_h(s)$ to represent the following action to take under Markov policy $\pi,\mu$. Next, we define the well-known value functions and Q-functions. we denote $V_h^{\pi,\mu}:\mS\rightarrow \mathbb{R}$ as the value function at step $h$ under Markov policy $\pi,\mu$ which calculates the expected cumulative rewards:
\[V_h^{\pi,\mu}(s):=\mathbb{E}_{\pi,\mu}\left[\sum_{h'=h}^{2H}r_{h'}(s_{h'},a_{h'})~:~s_h=s\right].\]
We also define $Q_h^{\pi,\mu}:\mS\times\mA$ as the Q-function at step $h$ so that $Q_{h}^{\pi,\mu}(s,a)$ calculates the cumulative rewards under policy $(\pi,\mu)$, starting from $(s,a)$ at step $h$:
\[Q_h^{\pi,\mu}(s,a):=\mathbb{E}_{\pi,\mu}\left[\sum_{h'=h}^{2H}r_{h'}(s_{h'},a_{h'})~:~s_h=s, a_h=a\right].\]
It's worth mentioned that the Markov policy of these two players are both deterministic, which means given a current state, their policies point to a specific action, instead of a probability distribution over all actions. That is because under the two-player turn based game setting, the Nash Equilibria is simply a pure strategy. 

For simplicity, we introduce the commonly-used notation of operator $\mathbb{P}_h$ which is $[\mathbb{P}_{h}(V)](s,a):=\mathbb{E}_{s'\sim\mathbb{P}_h(\cdot|s,a)}V(s')$ for any value function $V$. By this definition, we have the Bellman equation as follows:
\[Q_h^{\pi,\mu}(s,a)=(r_h+\mathbb{P}_h V_{h+1}^{\pi,\mu})(s,a)\]
holds for all $(s,a,h)\in\mS\times\mA\times [2H]$. We define $V_{2H+1}^{\pi,\mu}(s)=0$ for all $s\in\mS_{2H+1}$.

\paragraph{Best response and Nash equilibria}
Given the Markov policy $\pi$ of the max-player $P_1$, their exists a best response for the min-player, which we denote as $\mu^{\dagger}(\pi)$. It satisfies:
\[V_{h}^{\pi,\mu^{\dagger}(\pi)}(s)=\inf_{\mu}V_{h}^{\pi,\mu}(s)\]
holds for any $(s,h)\in\mS\times [2H]$. For brevity, we denote $V_{h}^{\pi,\dagger}:=V_{h}^{\pi,\mu^{\dagger}(\pi)}$. Similarly, we can also denote $\pi^{\dagger}(\mu), V_{h}^{\dagger,\mu}$ as the best response for the max-player and the corresponding value function. Furthermore, it is well known that there exists Markov policies $\pi^*,\mu^*$ which are the best responses of each other, and they satisfy:
\[V_h^{\pi^*,\dagger}(s)=\sup_{\pi}V_{h}^{\pi,\dagger}(s),~~~~~V_{h}^{\dagger,\mu^*}(s)=\inf_{\mu}V_{h}^{\dagger,\mu}(s)\]
holds for $\forall (s,h)\in\mS\times[2H]$.

\paragraph{Discounted 2-TBSG with Infinite Horizon} In such a game, the max-player $P_1$ and the min-player $P_2$ take turns to take action. However, different from standard 2-TBSG, there is only one episode in this game and there are infinite steps in this episode. 

\paragraph{2-TBSG with Linear Function Expression} 
A Markov Game $\mathrm{MG}(2H,\mS,\mA,\mathbb{P},r)$ is defined as linear when the probability transition kernels and the reward functions are linear with respect to a given feature map $\phi:\mS\times\mA\rightarrow \mathbb{R}^d$. Specifically, for each $h\in[2H]$, there exists an unknown vector $\mu_h\in\mathbb{R}^d$ and unknown measures $\theta_h=(\theta_h^{(1)},\theta_h^{(2)}, \ldots,\theta_h^{(d)})$ whose degree of freedom is $|S|\times d$, such that for $\forall (s,a)\in\mS\times\mA$:
\[\mathbb{P}_h(s'|s,a)=\langle\phi(s,a),\theta_h(s')\rangle~\text{and}~r(s,a)=\langle\phi(s,a), \mu_h\rangle.\]
For the complete version of the definition, we put it into Assumption \ref{assumption-linear} below.

\paragraph{Mathematical Notations} Let $f(n), g(n)$ be two positive series, we write $f(n)=\mathcal O(g(n))$ or $f(n)\lesssim g(n)$ if there exists a positive constant $C$ such that $f(n)\leqslant C\cdot g(n)$ for all $n$ larger than some $n_0\in\mathbb{N}$. Similarly, we write $f(n)=\Omega(g(n))$ or $f(n)\gtrsim g(n)$ if there exists a positive constant $C$ such that $f(n)\geqslant C\cdot g(n)$ for all $n$ larger than some $n_0\in\mathbb{N}$. If these two conditions hold simultaneously, we write $f(n) \asymp g(n)$ or $f(n)=\Theta(g(n))$. If logarithmic terms should be ignored, then we use the notations $\widetilde{\mathcal O}, \widetilde{\Omega}, \widetilde{\Theta}$. 

\begin{algorithm}[!t]
\caption{Optimistic Nash Q-learning on Two-player Turn-based Stochastic Games}
\hspace*{0.02in}{\bf Initialize:}
Let $\overline{Q}_h(s,a)\leftarrow 2H,~ \underline{Q}_h(s,a)\leftarrow 0$, and $N_h(s,a)\leftarrow 0$ for all $(s,a)\in\mS\times \mA$.\\
\hspace*{0.02in}{\bf Define:}
$\alpha_t=\frac{2H+1}{2H+t},~\iota\leftarrow \log(SAT^2)$.\\
\begin{algorithmic}[1]
\FOR{episode $k\in[K]$}
\STATE{observe the initial state $s_1$}
\FOR{step $h\in[2H]$}
\STATE Take action $a_h\leftarrow\arg\max_{a'\in\mA}\overline{Q}_{h}(s_h,a')$ if $h$ is an odd number, (i.e. $I(a_h)=1$),else take action $a_h\leftarrow\arg\min_{a'\in\mA}\underline{Q}_{h}(s_h,a')$. After that, observe the next state $s_{h+1}$.
\STATE $t=N_h(s_h,a_h)\leftarrow N_h(s_h,a_h)+1$.
\STATE $\beta_t\leftarrow c\sqrt{(2H)^3\iota/t}$.
\STATE $\overline Q_h(s_h,a_h)\leftarrow (1-\alpha_t)\cdot\overline Q_h(s_h,a_h)+\alpha_t\cdot [r_h(s_h,a_h)+\overline{V}_{h+1}(s_{h+1})+\beta_t]$.
\STATE $\underline Q_h(s_h,a_h)\leftarrow (1-\alpha_t)\cdot\underline Q_h(s_h,a_h)+\alpha_t\cdot[r_h(s_h,a_h)+\underline{V}_{h+1}(s_{h+1})-\beta_t]$.
\STATE $\overline{V}_h(s_h)\leftarrow \overline{Q}_h(s_h,a_h),~\underline{V}_h(s_h)\leftarrow \underline Q_h(s_h,a_h)$. 
\ENDFOR
\ENDFOR
\end{algorithmic}
\label{alg-1}
\end{algorithm}

\section{Regret Bound of Tabular 2-TBSG}
In 2-TBSG, we use a Nash Q-learning method: Algorithm \ref{alg-1} to obtain a policy tuple sequence $(\pi^k,\mu^k)$ to approximate a Nash equilibrium. In this algorithm, $\overline{Q}_h(s,a), \underline{Q}_h(s,a)$ are the upper and lower estimation of $Q_h^*(s,a)$ respectively. The max-player chooses action based on the upper estimation $\overline{Q}$ while the min-player chooses action based on the lower estimation $\underline{Q}$. By using the Upper Confidential Bound (UCB) technique, we can prove that
\[\underline{Q}_h(s,a)\leqslant Q_h^*(s,a)\leqslant \overline{Q}_h(s,a)\]
holds with high probability. The total regret of this algorithm is defined as:
\[\mathrm{Regret}(K)=\sum_{k=1}^{K}\left|\left(V_1^*-V_1^{\pi^k,\mu^k}\right)(s_1^k)\right|.\]
Here, $V_1^*$ is the abbreviation of $V_1^{\pi^*,\mu^*}$. In this section, we will focus on upper bounding the expected regret $\mathbb{E}[\mathrm{Regret}(K)]$. Notice that, unlike our related papers, there is an absolute value in our definition above, which is because under the two-player game setting, $\left(V_1^*-V_1^{\pi^k,\mu^k}\right)(s_1^k)$ can be either positive or negative. 
\begin{Theorem}[Main Theorem 1: Logarithmic Regret Bound of Q-learning for Episodic 2-TBSG]
\label{thm-main-1}
After using Algorithm \ref{alg-1}, the expected total regret for episodic two-player turn-based stochastic game (2-TBSG) can be upper bounded by:
\[\mathbb{E}[\mathrm{Regret}(K)]\leqslant \mathcal O\left(\frac{H^6SA\log(SAT)}{\gap^+_{\min}}\right).\]
\end{Theorem}
Here, the definition of $\gap^+_{\min}$ will be left in the following section. Notice that the theorem is not a straight-forward extension of existing works since Algorithm \ref{alg-1} is originally proposed by us and it is the first gap dependent logarithmic regret bound for 2-TBSG. In \cite{xie2020learning}, the authors have proposed a similar algorithm, but their setting is 2-player zero-sum simultaneous-move Markov Games and prove a $\mathcal O(\sqrt{d^3H^3T\iota^2})$ regret bound guarantee. The policies of the two players are mixed strategies, which makes it impossible to define gaps. Next, we provide the proof sketch of Theorem \ref{thm-main-1} and leave all the proof details to appendix.

\subsection{First Step: Split the Total Regret into the Expected Sum of Gaps}
In the first step, we split the total regret defined above into several single-step sub-optimality gaps. 
\begin{Lemma}\label{thm-split}
\begin{equation}
\left(V_1^*-  V_1^{\pi^k,\mu^k}\right)(s_1^k)=\mathbb{E}\left[\sum_{h=1}^{2H}\gap_h(s_h^k,a_h^k)~|~\pi^k,\mu^k\right].
\end{equation}
Here, $\gap_h(s_h^k,a_h^k):=V_h^*(s_h^k)-Q_h^*(s_h^k,a_h^k)$. 
\end{Lemma}
It is obvious to find that: when $a_h^k\in\mA_1$, which means $a_h^k$ is the max-player's action, then $\gap_h(s_h^k,a_h^k)\geqslant 0$. In contrast, when $a_h^k\in\mA_2$, which means $a_h^k$ is the min-player's action, then $\gap_h(s_h^k,a_h^k)\leqslant 0$. This is because the optimal strategy group $(\pi^*,\mu^*)$ stands for both the optimal point for max-player and min-player, though at different directions. Therefore, we introduce a new notation:
\begin{equation*}
\begin{aligned}
&\gap^{+}_h(s_h,a_h):=\left|\gap_h(s_h,a_h)\right|\\
&~~~=\begin{cases}V_h^*(s_h)-Q_h^*(s_h,a_h)&\mbox{if $a_h\in\mA_1$ (or $h$ is odd)}\\ Q_h^*(s_h,a_h)-V_h^*(s_h)&\mbox{if $a_h\in\mA_2$ (or $h$ is even)}\end{cases}.
\end{aligned}
\end{equation*}
Combining the two equations above, we obtain that:
\begin{align}\label{eqn-split2}
\mathbb{E}[\mathrm{Regret}(K)]&=\sum_{k=1}^{K}\left|\mathbb{E}\left[\sum_{h=1}^{2H}\gap_h(s_h^k,a_h^k)~|~\pi^k,\mu^k\right]\right|\notag\\
&\leqslant \sum_{k=1}^{K}\mathbb{E}\left[\sum_{h=1}^{2H}\gap^{+}_h(s_h^k,a_h^k)~|~\pi^k,\mu^k\right].
\end{align}

\subsection{Second Step: Concentration Inequalities and the Extension} 
According to an existing lemma in a related work \cite{bai2020provable}, we have the following conclusion as its special case in the 2-TBSG setting:
\begin{Lemma}[Lemma 11 of \cite{bai2020provable}]
For any $p\in(0,1]$, we let $\iota=\log(SAT/p)$. Then, with probability at least $1-p$, Algorithm 1 has the following property:
\begin{equation}\label{eqn-gap}
\begin{aligned}
&\overline{Q}_h^k(s,a)\geqslant Q_h^*(s,a)\geqslant \underline{Q}_h^k(s,a), \\
&\overline{V}_h^k(s)\geqslant V_h^*(s)\geqslant \underline{V}_h^k(s)
\end{aligned}
\end{equation}
holds for $\forall s\in\mS, h\in[2H], a\in\mA, k\in[K]$. We call the event above $\varepsilon_{\conc}$ and then $P(\varepsilon_{\conc})\geqslant 1-p$.
\end{Lemma}

This lemma indicates that $\overline{V},\overline{Q}$ and $\underline{V},\underline{Q}$ proposed by Algorithm 1 are exactly the upper bound and lower bound of the optimal point (Nash equilibrium) $V^*, Q^*$, with high probability. Similar to the episodic setting, we can make the following conclusion: 
\begin{equation}\label{eqn-5}
\gap^{+}_h(s_h^k,a_h^k)\leqslant\overline{Q}_h^k(s_h^k,a_h^k)-\underline{Q}_h^k(s_h^k,a_h^k)    
\end{equation}
holds for $\forall h\in[2H], k\in[K]$ as long as $\varepsilon_{\conc}$ holds. 
Then, Equation (\ref{eqn-5}) is transformed to:
\begin{equation}\label{eqn-6}
\gap^{+}_h(s_h^k,a_h^k)=\mathrm{clip}\left[\gap^+_h(s_h^k,a_h^k)~|~\gap^+_{\min}\right]\leqslant \mathrm{clip}\left[\overline{Q}_h^k(s_h^k,a_h^k)-\underline{Q}_h^k(s_h^k,a_h^k)~|~\gap^+_{\min}\right].
\end{equation}

\subsection{Third Step: Peeling}
Like \cite{yang2020q}, we separate all the gaps $\overline{Q}_h^k(s_h^k,a_h^k)-\underline{Q}_h^k(s_h^k,a_h^k)$ into different intervals and count them individually. Note that when the gap  $\overline{Q}_h^k(s_h^k,a_h^k)-\underline{Q}_h^k(s_h^k,a_h^k)$ belongs to $[0, \gap_{\min})$, then it will be clipped to 0 by Equation (\ref{eqn-6}). For the other gaps, we divide them into $N$ different intervals $[\gap_{\min}, 2\gap_{\min}), \cdots, [2^{N-1}\gap_{\min}, 2^{N}\gap_{\min})$. Here, $N=\lceil\log_{2}(2H/\gap_{\min})\rceil$. The following lemma tells us the upper bound of the counting number in each interval. 
\begin{Lemma}[Bounded Counting Number of Each Interval]\label{lemma-1} Under the concentration event $\varepsilon_{\conc}$, for each $n\in[N]$, we denote:
\[C^{(n)}:=\left|\left\{(k,h):\overline{Q}_h^k(s_h^k,a_h^k)-\underline{Q}_h^k(s_h^k,a_h^k)\in \Lambda_n\right\}\right|.\]
where $\Lambda_n=[2^{n-1}\gap_{\min}, 2^n\gap_{\min})$. Then, we have the following upper bound:
\[C^{(n)}\leqslant \mathcal O\left(\frac{H^6SA\iota}{4^n\gap_{\min}^2}\right)\]
where $\iota=\log(SAT/p)$ is the logarithmic term.
\end{Lemma}
Once we have this lemma proved, we can easily estimate the upper bound of expected cumulative regret.
\begin{align}
&\mathbb{E}[\mathrm{Regret}(K)]\leqslant \mathbb{E}\left[\sum_{k=1}^{K}\sum_{h=1}^{2H}\gap_h^{+}\left(s_h^k,a_h^k\right)\right]\notag\\
\leqslant& \sum_{\varepsilon_{\conc}}\mathbb{P}(\mathrm{traj})\cdot\sum_{k=1}^{K}\sum_{h=1}^{2H}\mathrm{clip}\left[(\overline{Q}_h^k-\underline{Q}_h^k)(s_h^k,a_h^k)~|~\gap_{\min}\right]+\sum_{\mathrm{traj}\notin\varepsilon_{\conc}}\mathbb{P}(\mathrm{traj})\cdot 2TH\notag\\
\leqslant& \sum_{n=1}^{N}2^n\gap_{\min}C^{(n)}+p\cdot 2TH\notag\\
\leqslant& \sum_{n=1}^{N}\mathcal O\left(\frac{H^6SA\iota}{2^n\gap_{\min}}\right)+p\cdot 2TH \leqslant \mathcal O\left(\frac{H^6SA\log(SAT)}{\gap_{\min}}\right).\notag
\end{align}
In the last step, we let $p=\frac{1}{T}$, and then $\iota=\log(SAT^2)=\mathcal O(\log(SAT))$. This leads to our main theorem. However, the proof of Lemma \ref{lemma-1} is difficult. We rely on a general lemma about the upper bound of the weighted sum of $(\overline{Q}_h^k-\underline{Q}_h^k)(s_h^k,a_h^k)$.

\begin{Lemma}[Peeling Argument]\label{lemma-2} Under the event $\varepsilon_{\conc}$, the following holds for $\forall h\in[2H]$ and a weight sequence $\{w_{k,h}\}_{k\in[K]}$ which satisfies: $0\leqslant w_{k,h}\leqslant w,~\sum_{k=1}^{K}w_{k,h}\leqslant C$, it holds that:
\begin{equation}\label{eqn-lemma2}
\sum_{k=1}^{K}w_{k,h}\left(\overline{Q}_h^k-\underline{Q}_h^k\right)(s_h^k,a_h^k)\leqslant 4ewSAH^2+60c\sqrt{SACewH^5\iota}.
\end{equation}
\end{Lemma}

After proving this lemma. We can make \[w_{k,h}^{n}=\mathbb{I}\left[\left(\overline{Q}_h^k-\underline{Q}_h^k\right)(s_h^k,a_h^k)\in\Lambda_n\right]\] and 
\[C_h^{(n)}=\sum_{k=1}^{K}\mathbb{I}\left[\left(\overline{Q}_h^k-\underline{Q}_h^k\right)(s_h^k,a_h^k)\in\Lambda_n\right],\]
where $\Lambda_n = [2^{n-1}\gap_{\min},2^{n}\gap_{\min})$. Then, the sequence $\{w_{k,h}^{n}\}_{k\in[K]}$ is a $\left(1, C_h^{(n)}\right)$-sequence. According to Lemma \ref{lemma-2},  we know that:
\begin{equation*}
2^{n-1}\gap_{\min}C_h^{(n)}\leqslant\sum_{k=1}^{K}w_{k,h}^{n}\left(\overline{Q}_h^k-\underline{Q}_h^k\right)(s_h^k,a_h^k)\leqslant 4eSAH^2+60c\sqrt{SAC_h^{(n)}eH^5\iota},
\end{equation*}
which leads to the fact that:
\begin{equation}
C_h^{(n)}\leqslant \mathcal O\left(\frac{H^5SA\iota}{4^n\gap_{\min}^2}\right).
\end{equation}
Finally, after summing them up:
\[C^{(n)}=\sum_{h=1}^{2H}C_h^{(n)}=\mathcal O\left(\frac{H^6SA\iota}{4^n(\gap^+_{\min})^2}\right),\]
which comes to our conclusion of Lemma \ref{lemma-1}. 

\paragraph{Connection Between the Vanilla Regret and Duality Regret.} Here we remark on the connection between the vanilla regret and duality regret. 
\begin{itemize}
\item In this paper, we analyze the vanilla regret $|V_1^*-V_1^{\pi^k, \mu^k}|$. Since the game setting we are analyzing is 2-player turn-based stochastic game (2-TBSG), and $(\pi^*,\mu^*)$ is a Nash Equilibrium rather than the optimal policy. The vanilla regret actually measures the distance between the two value functions proposed by $(\pi^*, \mu^*)$ and $(\pi^k, \mu^k)$. 
\item Another type of regret is the duality regret $V_1^{\dagger, \mu^k}-V_1^{\pi^k, \dagger}$. It can measure how close the policy $(\pi^k,\mu^k)$ is to a Nash Equilibrium. 
\end{itemize}

Therefore, it's an important question whether a small vanilla regret implies a small duality regret. (Notice that the inverse would not be true since there might be more than one Nash Equilibria.) To answer this question, we propose the following proposition. 

\begin{Proposition}
\label{prop-1}
If for $\forall s\in\mS, h\in [2H]$, the vanilla regret
\[\left|V_h^*(s)-V_h^{\pi,\mu}(s)\right|<\frac{1}{2}\cdot\gap_{\min},\]
we can conclude that $(\pi,\mu)$ is a Nash Equilibrium, which means its duality regret is 0. 
\end{Proposition}

\begin{proof}[Proof of Proposition \ref{prop-1}] We are going to prove $(\pi_h, \mu_h)=(\pi^*_h,\mu^*_h)~\forall h\in [2H]$ by using the method of induction. When $h=2H$, notice that for $\forall s\in\mS$:
\[V_{2H}^{\pi,\mu}(s)= r(s, \mu_{2H}(s))=Q_{2H}^*(s, \mu_{2H}(s)).\]
Since $\frac{1}{2}\cdot\gap_{\min}>\left|V_{2H}^*(s)-V_{2H}^{\pi,\mu}(s)\right|=\left|V_{2H}^*(s)-Q_{2H}^*(s,\mu_{2H}(s))\right|$, therefore we have $\mu_{2H}^*(s)=\mu_{2H}(s)$ according to the definition of $\gap_{\min}^+$. It means that $(\pi,\mu)$ and $(\pi^*,\mu^*)$ make the same decisions at the $2H$'s step. On the other hand, assume that $(\pi_t,\mu_t)=(\pi_t^*,\mu_t^*)$ holds for $t = h+1, h+2, \ldots, 2H$, then for $t=h$: if $h$ is an odd number, then it's max-player's turn to take action, and it holds that for $\forall s\in\mS$:
\[V_{h}^{\pi,\mu}(s)=r(s,\pi(s))+\mathbb{E}_{s'|s,\pi(s)} V_{h+1}^{\pi,\mu}(s') \overset{(a)}{=} r(s,\pi(s))+\mathbb{E}_{s'|s,\pi(s)} V_{h+1}^*(s') = Q_h^*(s,\pi_h(s)).\]
Since $\frac{1}{2}\cdot\gap_{\min}>\left|V_h^*(s)-V_h^{\pi,\mu}(s)\right|=\left|V_h^*(s)-Q_h^*(s,\pi_h(s))\right|$, we can conclude that $\pi_h(s)=\pi^*_h(s)$ according to the definition of $\gap_{\min}^+$. Therefore, $(\pi,\mu)$ and $(\pi^*,\mu^*)$ also make the same decisions at the $h$'s step, which finishes our induction. When $h$ is an even number, we can finish our induction in the same way, and that comes to the conclusion of Proposition \ref{prop-1} since $(\pi,\mu)=(\pi^*,\mu^*)$ and then the policy pair $(\pi,\mu)$ is a Nash Equilibrium. It leads to fact that:
\[V_1^{\dagger, \mu}=V_1^{\pi,\mu}=V_1^{\pi,\dagger},\]
which means policy pair $(\pi,\mu)$ has zero duality regret. 

\end{proof}

The conclusion shows that when vanilla regret is sufficiently small for all $h\in [2H], s\in\mS$, then the policy pair is a Nash Equilibrium. It bridges the gap between the vanilla regret and the duality regret, and makes it more reasonable and convincing for us to work on the upper bound for the expected sum of vanilla regrets.

\section{Regret Bound of Discounted 2-TBSG}
In this section, we study the discounted 2-TBSG with infinite horizon. In order to obtain a policy tuple sequence $(pi^k,\mu^k)$ to approximate a Nash equilibrium, we propose the following Algorithm \ref{alg-2}, which is similar to Algorithm \ref{alg-1}.
We also use the UCB technique to establish a upper estimation $\hat{Q}(s,a)$ and a lower estimation $\breve{Q}(s,a)$ of the optimal Q-function $Q^*(s,a)$.

\begin{algorithm}[!t]
\caption{Optimistic Nash Q-learning on Discounted 2-TBSG with Infinite Horizon}
\hspace*{0.02in}{\bf Initialize:}
Let $\overline{Q}(s,a), \hat{Q}(s,a)\leftarrow 1/(1-\gamma)$ and  $\underline{Q}(s,a), \breve{Q}(s,a)\leftarrow 0$ for all $(s,a)\in\mS\times \mA$. Also $N(s,a)\leftarrow 0$.\\
\hspace*{0.02in}{\bf Define:}
$\iota(k)=\log(SAT(k+1)(k+2)), \alpha_k=\frac{H+1}{H+k}$ where $H=\frac{\log(2/(1-\gamma)\gap^+_{\min})}{\log(1/\gamma)}$.\\
\begin{algorithmic}[1]
\STATE{Observe the initial state $s_1$.}
\FOR{episode $t\in[T]$}
\STATE Take action $a_t\leftarrow\arg\max_{a'\in\mA}\overline{Q}(s_t,a')$ if $t$ is an odd number, (i.e. $I(a_t)=1$), else take action $a_t\leftarrow\arg\min_{a'\in\mA}\underline{Q}(s_t,a')$. After that, observe the reward $r(s_t,a_t)$ as well as the next state $s_{t+1}$.
\STATE $k=N(s_t,a_t)\leftarrow N(s_t,a_t)+1$.
\STATE $b_k\leftarrow \frac{c_2}{1-\gamma}\sqrt{H\iota(k)/k}$,~~~~~~~Here $c_2$ is a constant that can be set to $4\sqrt{2}$.
\STATE $\overline Q(s_t,a_t)\leftarrow (1-\alpha_k)\cdot\overline Q(s_t,a_t)+\alpha_k\cdot[r(s_t,a_t)+\gamma\hat{V}(s_{t+1})+b_k]$.
\STATE $\underline Q(s_t,a_t)\leftarrow (1-\alpha_k)\cdot\underline Q(s_t,a_t)+\alpha_k\cdot[r(s_t,a_t)+\gamma\breve{V}(s_{t+1})-b_k]$. 
\STATE $\hat{Q}(s_t,a_t)\leftarrow \min(\hat{Q}(s_t,a_t), \overline{Q}(s_t,a_t)), ~\breve{Q}(s_t,a_t)\leftarrow \max(\breve{Q}(s_t,a_t), \underline{Q}(s_t,a_t)) $.
\STATE $\hat{V}(s_t)\leftarrow \hat{Q}(s_t,a_t), \breve{V}(s_t)\leftarrow \breve{Q}(s_t,a_t)$.
\ENDFOR
\end{algorithmic}
\label{alg-2}
\end{algorithm}

Since the V-function denotes the expected discounted sum of rewards given the initial state $s$ and the Q function denotes the expected discounted sum of rewards given the initial state $s$ and the initial action $a$, so they can be described as:
\begin{equation*}
\begin{aligned}
V^{\pi}(s) &:= \mathbb{E}\left[\sum_{i=1}^{\infty}\gamma^{i-1}\cdot r(s_i, a_i)~:~s_1=s\right],\\
Q^{\pi}(s,a) &:=\mathbb{E}\left[\sum_{i=1}^{\infty}\gamma^{i-1}\cdot r(s_i, s_i)~:~s_1=s, a_1 = a\right].
\end{aligned}
\end{equation*}
where $\gamma$ is the discounted ratio. In the following statements, $\overline{Q}_t, \overline{V}_t$ stand for the $\overline{Q}, \overline{V}$ functions in the $t$-th iteration, and so does the other subscripts.

\subsection{Sub-optimality Gap and the Splitting of Total Regret}
Similar to the episodic setting, given $(s,a)\in\mS\times\mA$, define $\gap(s,a)$ as:
\[\gap(s,a):=V^*(s)-Q^*(s,a).\]
Notice that, when $a\in\mA_1$, which means the action $a$ is taken by max-player $P_1$, then $V^*(s)\geqslant Q^*(s,a)\Rightarrow \gap(s,a)\geqslant 0$. In contrast, when $a\in\mA_2$, which means the action $a$ is taken by min-player $P_2$, then $V^*(s)\leqslant Q^*(s,a)\Rightarrow\gap(s,a)\leqslant 0$. Here, we can introduce a notation $\gap^+(s,a):=|\gap(s,a)|$ which stands for the absolute value of $\gap(s,a)$. Also, we denote $\gap^+_{\min}$ as the minimum non-zero absolute gap:
\[\gap^+_{\min}:=\min_{s,a}\{\gap^+(s,a)~:~\gap^+(s,a)\neq 0\} > 0.\]
In the main theorem proposed in the following section, we will estimate an upper bound of the expected total regret for the first $T$ steps
\[\mathrm{Regret}(T):=\sum_{t=1}^{T}\left|(V^*-V^{\pi_t, \mu_t})(s_t)\right|.\]

\subsection{Main Theorem}
In this section, we propose our main theorem in the infinite-horizon setting. Unlike dual gap regret \cite{xie2020learning}, the total regret above proposed by \cite{liu2020regret} follows the sample complexity definition in \cite{sm2003on} and directly compares the actual value function and the value function from the first $t$ iterations. Similar to the episodic setting, we can obtain a gap-dependent logarithmic upper bound for the expected total regret. 
\begin{Theorem}[Main Theorem 2: Logarithmic Regret Bound of Q-learning for Infinite-horizon Discounted 2-TBSG] After using Algorithm \ref{alg-2}, the expected total regret for infinite-horizon two-player turn-based stochastic game can be upper bounded by:
\begin{equation*}
\mathbb{E}[\mathrm{Regret}(T)]\leqslant \mathcal O\left(\frac{SA}{\gap^+_{\min}(1-\gamma)^5\log(1/\gamma)}\cdot\log\frac{SAT}{\gap^+_{\min}(1-\gamma)}\right).
\end{equation*}
\label{thm-main-2}
\end{Theorem}
Next, we introduce the proof sketch of Theorem \ref{thm-main-2}.

\subsection{First step: Splitting the Regret into Expected Sum of Gaps} \label{sec:split_regret}
The splitting of the regret is just the same as that in the finite horizon setting:
\begin{align}\label{eqn-7}
\mathbb{E}[\mathrm{Regret}(T)]&=\mathbb{E}\left[\sum_{t=1}^{T}\left|\sum_{h=0}^{\infty}\gamma^h\gap(s_{t+h},a_{t+h})\right|\right]\notag\\
&\leqslant\mathbb{E}\left[\sum_{t=1}^{T}\sum_{h'=t}^{+\infty}\gamma^{h'-t}\gap^+(s_{h'},a_{h'})\right].
\end{align}    
Notice that when $t$ is an odd number, it's the max-player's turn to take action, so $\gap(s_t,a_t)\geqslant 0$ and then:
\begin{equation*}
\begin{aligned}
&~~\gap^+(s_t,a_t)=\gap(s_t,a_t)=V^*(s_t)-Q^*(s_t,a_t)\\
&=Q^*(s_t,a^*)-Q^*(s_t,a_t)\leqslant \hat{Q}_t(s_t,a^*)-\breve{Q}_t(s_t,a_t)\\
&\leqslant \left(\hat{Q}_t-\breve{Q}_t\right)(s_t,a_t).
\end{aligned}    
\end{equation*}
Similarly, when $t$ is an even number, it's the min-player's turn to take action, so $\gap(s_t,a_t)\leqslant 0$, and then:
\begin{equation*}
\begin{aligned}
&~~\gap^+(s_t,a_t)=-\gap(s_t,a_t)=Q^*(s_t,a_t)-V^*(s_t)\\
&=Q^*(s_t,a_t)-Q^*(s_t,a^*)\leqslant \hat{Q}_t(s_t,a_t)-\breve{Q}_t(s_t,a^*)\\
&\leqslant \left(\hat{Q}_t-\breve{Q}_t\right)(s_t,a_t).
\end{aligned}    
\end{equation*}
Therefore, we can conclude that $\gap^+(s_t,a_t)\leqslant \left(\hat{Q}_t-\breve{Q}_t\right)(s_t,a_t)$. By the definition of $\gap^+_{\min}$, we have:
\begin{equation*}
\gap^+(s_t,a_t)=\clip[\gap^+(s_t,a_t)|\gap^+_{\min}]\leqslant \clip\left[\left(\hat{Q}_t-\breve{Q}_t\right)(s_t,a_t)|\gap^+_{\min}\right].
\end{equation*}
Combine it with Equation (\ref{eqn-7}), we obtain that:
\begin{equation}
\label{eqn-8}
\mathbb{E}[\mathrm{Regret}(T)]\leqslant \mathbb{E}\left[\sum_{t=1}^{T}\sum_{h'=t}^{+\infty}\gamma^{h'-t}\cdot\clip\left[\left(\hat{Q}_{h'}-\breve{Q}_{h'}\right)(s_{h'},a_{h'})|\gap^+_{\min}\right]\right].
\end{equation}

\subsection{Second step: Concentration Properties}
Extended from \cite{dong2019q}, we can obtain the following lemma which shows that Algorithm \ref{alg-2} satisfies bounded learning error with high probability.
\begin{Lemma}[Concentration Property]
\label{lemma-3}
When applying Algorithm \ref{alg-2}, event $\mathcal E_{\conc}$ occurs with probability at least $1-\frac{1}{T}$. Here, $\mathcal E_{\conc}$ occurs if  $\forall (s,a,t)\in\mS\times\mA\times\mathbb{N}$:
\begin{equation*}
\begin{aligned}
0 &\leqslant \left(\hat{Q}_t-Q^*\right)(s,a)\leqslant \left(\overline{Q}_t-Q^*\right)(s,a)\leqslant \frac{\alpha_{n^t}^0}{1-\gamma}+\sum_{i=1}^{n^t}\gamma\alpha_{n^t}^i\left(\hat{V}_{\tau(s,a,i)}-V^*\right)(s_{\tau(s,a,i)})+\beta_{n^t},\\
0 &\leqslant \left(Q^*-\breve{Q}_t\right)(s,a)\leqslant \left(Q^*-\underline{Q}_t\right)(s,a)\leqslant \frac{\alpha_{n^t}^0}{1-\gamma}+\sum_{i=1}^{n^t}\gamma\alpha_{n^t}^i\left(V^*-\breve{V}_{\tau(s,a,i)}\right)(s_{\tau(s,a,i)})+\beta_{n^t}.
\end{aligned}    
\end{equation*}
Here, $\iota(k)=\log(SAT(k+1)(k+2))$ and $\beta_k=\frac{c_3}{1-\gamma}\sqrt{\frac{H\iota(k)}{k}}$.
\end{Lemma}

Then, under $\mathcal E_{\conc}$, we know that:
\begin{equation}\label{eqn-inf-conc}
0\leqslant \left(\hat{Q}_t-\breve{Q}_t\right)(s,a)\leqslant \frac{2\alpha_{n^t}^0}{1-\gamma}+2\beta_{n^t}+\sum_{i=1}^{n^t}\gamma\alpha_{n^t}^i\left(\hat{V}_{\tau(s,a,i)}-\breve{V}_{\tau(s,a,i)}\right)(s_{\tau(s,a,i)}).
\end{equation}

\subsection{Third step: Peeling} \label{sec:peeling} 
Similar to Lemma \ref{lemma-2}, we upper bound the weighted sum of upper-lower gaps and then bound the counting number of gaps in different intervals, just like we did in the episodic setting. 
\begin{Lemma}[Peeling Argument]
\label{lemma-4}
Under the event $\mathcal E_{\conc}$, the following holds for any weighted sequence $\{\omega_t\}$ which satisfies: $0\leqslant \omega_t\leqslant \omega, \sum_{t=1}^{+\infty}\omega_t\leqslant C$.
\begin{equation*}
\sum_{t=1}^{+\infty}\omega_t\left(\hat{Q}_t-\breve{Q}_t\right)(s_t,a_t)\leqslant \frac{\gamma^H C}{1-\gamma}+\mathcal O\left(\frac{\sqrt{\omega SAHC\iota(C)}+\omega SA}{(1-\gamma)^2}\right).
\end{equation*}
\end{Lemma}

After that, we classify the positive gaps into different intervals. Since all the positive gaps $\gap^+(s,a)\in[\gap^+_{\min}, 1/(1-\gamma))$ and this interval can be separated into $N$ intervals \[\Lambda_n=[2^{n-1}\gap^+_{\min}, 2^n\gap^+_{\min})\]
where $n=1,2,\ldots,N$. Here, $N=\left\lceil\log_2\left(\frac{1}{\gap^+_{\min}(1-\gamma)}\right)\right\rceil$. Under the event $\mathcal E_{\conc}$, for $n\in[N]$, we define:
\[C^{(n)}:=\left|\left\{t\in\mathbb{N}_{+}~: \left(\hat{Q}_t-\breve{Q}_t\right)(s_t,a_t)\in\Lambda_n \right\}\right|.\]
By using the sequence 
\[\omega_t^{(n)}:=\mathbb I\left[\left(\hat{Q}_t-\breve{Q}_t\right)(s_t,a_t)\in \Lambda_n\right],\]
we can upper bound the $C^{(n)}$ by using Lemma \ref{lemma-4}.
\begin{Lemma}
\label{lemma-5}
We can upper bound the $C^{(n)}$ by:
\[\mathcal O\left(\frac{SA}{4^n(\gap^+_{\min})^2(1-\gamma)^4\log(1/\gamma)}\log\left(\frac{SAT}{(1-\gamma)\gap^+_{\min}}\right)\right).\]
\end{Lemma}


Finally, we come to our main theorem. According to Equation (\ref{eqn-8}), we know that: if the trajectory satisfies the $\mathcal E_{\conc}$ condition, then:
\begin{align}
\label{eqn-9}
&~~\mathrm{Regret}(T)\leqslant \frac{1}{1-\gamma}\sum_{t=1}^{+\infty}\clip\left[\left(\hat{Q}_t-\breve{Q}_t\right)(s_t,a_t)|\gap^+_{\min}\right]\leqslant \frac{1}{1-\gamma}\sum_{n=1}^{N}2^n\gap^+_{\min}C^{(n)}\notag\\
&\leqslant \sum_{n=1}^{N} \mathcal O\left(\frac{SA}{2^n \gap^+_{\min}(1-\gamma)^5\log(1/\gamma)}\cdot\iota\right)= \mathcal O\left(\frac{SA}{\gap^+_{\min}(1-\gamma)^5\log(1/\gamma)}\cdot \iota\right).
\end{align}
Here, $\iota=\log\left(\frac{SAT}{\gap^+_{\min}(1-\gamma)}\right)$ is the logarithmic term on $T$. For the other trajectories outside $\mathcal E_{\conc}$, we have a trivial upper bound:
\begin{equation}
\label{eqn-10}
\mathrm{Regret}(T)\leqslant \sum_{t=1}^{T}\sum_{h'=t}^{+\infty}\gamma^{h'-t}\cdot\left(\hat{Q}_{h'}-\breve{Q}_{h'}\right)(s_{h'},a_{h'})\leqslant \sum_{t=1}^{T}\sum_{h'=t}^{+\infty}\frac{\gamma^{h'-t}}{1-\gamma}\leqslant \frac{T}{(1-\gamma)^2}.
\end{equation}    
Now we combine Equation (\ref{eqn-9}) and Equation (\ref{eqn-10}), we obtain that:
\begin{align}
\label{eqn-11}
&\mathbb{E}[\mathrm{Regret}(T)]=\mathbb{E}\left[\sum_{t=1}^{T}\sum_{h'=t}^{+\infty}\gamma^{h'-t}\gap(s_{h'},a_{h'})\right]\notag\\
&\leqslant\mathbb{P}(\overline{\mathcal E_{\conc}})\cdot\frac{T}{(1-\gamma)^2}+\mathbb{P}(\mathcal E_{\conc})\cdot \mathcal O\left(\frac{SA}{\gap^+_{\min}(1-\gamma)^5\log(1/\gamma)}\cdot\iota\right)\notag\\
&\leqslant \mathcal O\left(\frac{SA}{\gap^+_{\min}(1-\gamma)^5\log(1/\gamma)}\cdot\iota\right).
\end{align}
which comes from $\mathbb{P}(\overline{\mathcal E_{\conc}})\leqslant 1/T$. Theorem \ref{thm-main-2} is proved and it provides us an upper bound which is logarithmically dependent on $T$.

\section{Episodic 2-TBSG with Linear Function Approximation}
In this section, we analyze the gap dependent total regret bound of episodic 2-TBSG under the linear function expression assumption. Here, we have two different settings, centralized version and independent version. As follows, we are going to introduce these two settings and our corresponding algorithms one by one. After that, gap dependent regret bound will be provided in both settings.

\subsection{Centralized Setting and Algorithm}
\begin{Assumption}[Assumption 4.1 from \cite{he2020logarithmic}]
\label{assumption-linear}
A Markov Game $\mathrm{MG}(2H,\mS,\mA,\mathbb{P},r)$ is defined as linear when the probability transition kernels and the reward functions are linear with respect to a given feature map $\phi:\mS\times\mA\rightarrow \mathbb{R}^d$ where $d$ is the feature dimension. Specifically, for each $h\in[2H]$, there exists an unknown vector $\mu_h\in\mathbb{R}^d$ and unknown measures $\theta_h=(\theta_h^{(1)},\theta_h^{(2)}, \ldots,\theta_h^{(d)})$ whose degree of freedom is $|S|\times d$, such that for $\forall (s,a)\in\mS\times\mA$:
\[\mathbb{P}_h(s'|s,a)=\langle\phi(s,a),\theta_h(s')\rangle~\text{and}~r(s,a)=\langle\phi(s,a), \mu_h\rangle.\]
For simplicity, we assume that $\|\phi(s,a)\|_2\leqslant 1, \|\mu_h\|_2\leqslant \sqrt{d}$ and $\|\theta_h(\mS)\|\leqslant \sqrt{d}$.
\end{Assumption}

In the centralized setting, a central controller controls both players, and this central controller's goal is to learn a Nash Equilibrium. In our algorithm, both the max-player and the min-player update their policies $(\pi^k,\mu^k)$ according to the history information. Under Assumption \ref{assumption-linear}, we know that for any policy $\pi$, the action-value function $Q_h^{\pi}(s,a)$ is a linear function $\langle\phi(s,a), \theta_h^{\pi}\rangle$ with respect to the feature $\phi(s,a)$. In order to estimate the optimal action-value function $Q_h^*(s,a):=\langle\phi(s,a), \theta_h^*\rangle$, we only have to estimate the parameters $\theta_h^*$. Since Assumption \ref{assumption-linear} only gives a condition on the linear structure of the stochastic game, it is still a two-player turn-based general sum stochastic game with finite horizon. Therefore, we have exactly the same definition on the cumulative regret (or total regret) as Section 4:
\[\mathrm{Regret}(K)=\sum_{k=1}^{K}\left|\left(V_1^*-V_1^{\pi^k,\mu^k}\right)(s_1^k)\right|.\]


In the following Least Square Value Iteration on 2-TBSG (LSVI-2TBSG) algorithm, we introduce two new variables $\overline{w}_h^k, \underline{w}_h^k$, which are the upper and lower estimations of $\theta_h^*$ in the $k$-th episode. They are computed by solving the following regularized least-square problems:
\begin{equation}
\begin{aligned}
\overline{w}_h^k &\leftarrow \arg\min_{w\in\mathbb{R}^d} \lambda\|w\|^2 + F_1(w),\\
\underline{w}_h^k &\leftarrow \arg\min_{w\in\mathbb{R}^d} \lambda\|w\|^2 + F_2(w), 
\end{aligned}    
\end{equation}
where
\[F_1(w)=\sum_{i=1}^{k-1}\left[\phi(s_h^i, a_h^i)^{\top}w-r_h(s_h^i,a_h^i)-\overline{V}_{h+1}^k(s_{h+1}^i)\right]^2,\]
\[F_2(w)=\sum_{i=1}^{k-1}\left[\phi(s_h^i, a_h^i)^{\top}w-r_h(s_h^i,a_h^i)-\underline{V}_{h+1}^k(s_{h+1}^i)\right]^2.\]
Actually, these two least-square problems can be easily solved as:
\begin{equation*}
\begin{aligned}
\overline{w}_h^k&=(\Lambda_h^k)^{-1}\sum_{i=1}^{k-1}\phi(s_h^i,a_h^i)\left[r_h(s_h^i,a_h^i)+\overline{V}_{h+1}^k(s_{h+1}^i)\right],\\
\underline{w}_h^k&=(\Lambda_h^k)^{-1}\sum_{i=1}^{k-1}\phi(s_h^i,a_h^i)\left[r_h(s_h^i,a_h^i)+\underline{V}_{h+1}^k(s_{h+1}^i)\right],
\end{aligned}    
\end{equation*}
where $\Lambda_h^k=\sum_{i=1}^{k-1}\phi(s_h^i,a_h^i)\phi(s_h^i,a_h^i)^{\top}+\lambda I$. Then, we update the estimated Q-values by:
\begin{equation*}
\begin{aligned}
\overline Q_h^k(s,a)&= \min(2H, \phi(s,a)^{\top}\overline{w}_h^k + \beta\cdot T(s,a)),\\
\underline Q_h^k(s,a)&= \max(0, \phi(s,a)^{\top}\underline{w}_h^k - \beta\cdot T(s,a)), 
\end{aligned}    
\end{equation*}
where $T(s,a)=\sqrt{\phi(s,a)^{\top}(\Lambda_h^k)^{-1}\phi(s,a)}$ can be regarded as a UCB term. For the pseudo-code of LSVI-2TBSG algorithm, we leave it to the appendix.

\subsection{Main Theorem for the Centralized Setting} 
In this section, we propose the main theorem under the linear function expression assumption in the centralized setting. When using LSVI-2TBSG algorithm, the expected total regret can be upper bounded by the following theorem:
\begin{Theorem}[Main Theorem 3: Logarithmic Regret Bound of LSVI-2TBSG (Centralized)]
Under Assumption \ref{assumption-linear}, after using LSVI-2TBSG algorithm (centralized)
\[\mathbb{E}[\mathrm{Regret}(K)]\leqslant 1+\frac{Cd^3H^5\log(16dK^2(K+1)H^3)}{\gap^+_{\min}}\iota,\]
where $\iota=\log\left(\frac{Cd^3H^4\log(4dKH)}{(\gap^+_{\min})^2}\right)$ is a logarithmic term.
\label{thm-main-3}
\end{Theorem}
By using similar techniques, we can prove the theorem and propose the first gap-dependent logarithmic regret bound under the linear MDP assumption. We leave the technical proof of this theorem to the appendix.

\subsection{Independent Setting and Algorithm}
In the independent setting, we do not have a central controller who controls both players. We can only control the max-player and play against the min-player whose policies are arbitrary but potentially adversarial. Since we only control the max-player, our goal is not to learn a Nash Equilibrium, but to maximize the reward of the max-player. Because of the differences between the centralized setting and the independent setting, we are going to redefine the gaps and regret functions in this section. 

Since we can not get access to the min-player's policies and the Markov model of the game a priori, we are interested in the exploitability of max-player:
\[\mathrm{Explicit}(\pi^k, \mu^k):=V_1^{\dagger, \mu^k}(s_1^k)-V_1^{\pi^k, \mu^k}(s_1^k),\]
which measures how much better the max-player can perform. Then, the cumulative regret can be defined as the sum of exploitability in different episodes:
\[\mathrm{Regret}_{\mu}(K):=\sum_{k=1}^{K}\mathrm{Explicit}(\pi^k, \mu^k)=\sum_{k=1}^{K}\left(V_1^{\dagger, \mu^k}(s_1^k)-V_1^{\pi^k, \mu^k}(s_1^k)\right).\]

Also, we need to redefine the gap. Previously, the gap is defined as $\gap_h(s,a)=|V_h^*(s)-Q^*_h(s,a)|$ and $\gap_{\min}^+:=\min_{h,s,a}\{\gap_h(s,a)>0\}$. However, in the independent setting, we can not control the min-player so it is not suitable to only consider $(\pi,\mu)=(\pi^*,\mu^*)$ since the Nash Equilibrium point is our final target. In order to measure the gap caused by max-player, we define:
\[\gap^{\mu}_h(s,a)=|V_h^{\dagger, \mu}(s)-Q_h^{\dagger, \mu}(s,a)|.\]
Here, we still need the absolute value since $V_h^{\dagger, \mu}(s)-Q_h^{\dagger, \mu}(s,a)$ is non-negative when $h$ is an odd number (and it's the max-player's turn to take action) while non-positive when $h$ is an even number. The minimal gap can be obtained after taking minimum over all the $(h,s,a)$ tuples and all possible pure strategy $\mu$:
\[\gap_{\min}^+:=\min_{\mu,h,s,a}\{\gap^{\mu}_h(s,a)>0\}.\]
Notice that the total number of pure strategies is finite, so the minimal gap above is positive and well-defined. Similar as the LSVI-2TBSG algorithm, we introduce a new variable $w_h^k$, which is the upper estimation of $\theta_h^*$ in the $k$-th episode. Here, we do not need the lower estimation $\underline{w}_h^k$ since the min-player's policy is beyond our control. In each episode, the $w^k_h$ is computed by solving the following regularized least-square problem:
\[w_h^k\leftarrow \arg\min_{w\in\mathbb{R}^d} \lambda\|w\|^2 + \sum_{i=1}^{k-1}\left[\phi(s_h^i,a_h^i)^{\top}w-r_h(s_h^i,a_h^i)-\overline{V}_{h+1}^k(s_{h+1}^i)\right]^2. \]
Actually, it can be solved as:
\[w_h^k=(\Lambda_h^k)^{-1}\sum_{i=1}^{k-1}\phi(s_h^i,a_h^i)\left[r_h(s_h^i,a_h^i)+\overline{V}_{h+1}^k(s_{h+1}^i)\right], \]
where $\Lambda_h^k=\sum_{i=1}^{k-1}\phi(s_h^i,a_h^i)\phi(s_h^i,a_h^i)^{\top}+\lambda I$. Then, we update the estimated Q-values by:
\[\overline Q_h^k(s,a)= \min(2H, \phi(s,a)^{\top}\overline{w}_h^k + \beta\cdot T(s,a)),\]
where $T(s,a)=\sqrt{\phi(s,a)^{\top}(\Lambda_h^k)^{-1}\phi(s,a)}$ is the UCB term. We leave the pseudo-code of the algorithm to the appendix.

\subsection{Main Theorem for the Independent Setting} 
In this section, we propose the main theorem under the linear function expression assumption in the independent setting. The expected cumulative regret can be upper bounded by a logarithmic term:
\begin{Theorem}[Main Theorem 4: Logarithmic Regret Bound of LSVI-2TBSG (Independent)]
Under Assumption \ref{assumption-linear}, after using LSVI-2TBSG algorithm (independent)
\[\mathbb{E}[\mathrm{Regret}(K)]\leqslant 1+\frac{Cd^3H^5\log(16dK^2(K+1)H^3)}{\gap^+_{\min}}\iota,\]
where $\iota=\log\left(\frac{Cd^3H^4\log(4dKH)}{(\gap^+_{\min})^2}\right)$ is a logarithmic term.
\label{thm-main-4}
\end{Theorem}

We provide a logarithmic regret upper bound for 2-TBSG under the linear MDP assumption, in both centralized and independent settings. To the best of our knowledge, this is the very first gap dependent upper bound for 2-TBSG, which makes our results novel and complete. 

\section{Conclusion}
We gave the first set of gap-dependent logarithmic regret bounds for two-player turn-based stochastic Markov games in both tabular and episodic cases, and in both centralized setting and independent setting. 
One fruitful future direction is to extend our analysis to more general settings~\cite{wang2020reinforcement,jiang2017contextual,du2021bilinear,jin2021bellman}.

\bibliography{reference}
\bibliographystyle{alpha}

\newpage

\appendix
\section{Proof of Lemma \ref{lemma-2}}

According to the update rule of Algorithm \ref{alg-1}, we can get the following equations: let $t=N_{h}^{k}(s,a)$ be the total times when state-action tuple $(s,a)$ appears at the $h$-th step in the first $k-1$ episodes. Suppose tuple $(s,a)$ previously appeared at episodes $k^1,k^2,\cdots,k^t<k$ at the $h$-th step. Denote $\tau_h(s,a,i):=k^i$. Then, it holds that:
\begin{equation}
\begin{aligned}
&\overline{Q}_h^k(s,a)=\alpha_t^0\cdot 2H+\sum_{i=1}^{t}\alpha_t^i\left[r_h(s,a)+\overline{V}_{h+1}^{k^i}(s_{h+1}^{k^i})+\beta_i\right],\\
&\underline{Q}_h^k(s,a)=\sum_{i=1}^{t}\alpha_t^i\left[r_h(s,a)+\underline{V}_{h+1}^{k^i}(s_{h+1}^{k^i})-\beta_i\right].
\end{aligned}
\end{equation}
Then, we can upper bound the weighted sum of upper-lower gaps.
\begin{align}
&\sum_{k=1}^{K}w_{k,h}\left(\overline{Q}_h^k-\underline{Q}_h^k\right)(s_h^k,a_h^k)\notag\\
\leqslant & \sum_{k=1}^{K}w_{k,h}\left(\alpha_{n_h^k}^{0}\cdot 2H+\sum_{i=1}^{n_h^k}\alpha_{n_h^k}^{i}\left(\overline{V}_{h+1}^{\tau_h(s,a,i)}-\underline{V}_{h+1}^{\tau_h(s,a,i)}\right)(s_{h+1}^{\tau(s,a,i)})+2\beta_{n_h^k}\right)\notag\\
=& \sum_{k\leqslant K, n_h^k=0}w_{k,h}\cdot 2H + \sum_{k=1}^{K}2w_{k,h}\beta_{n_h^k}+\sum_{k=1}^{K}\left(\overline{V}_{h+1}^{k}-\underline{V}_{h+1}^{k}\right)(s_{h+1}^k)\left(\sum_{i=n_h^k+1}^{N_h^k(s_h^k,a_h^k)}\alpha_i^{n_h^k}w_{\tau_h(s_h^k,a_h^k,i),h}\right)\notag\\
\leqslant &~2SAHw+\sum_{k=1}^{K}2w_{k,h}\beta_{n_h^k}+\sum_{k=1}^{K}w_{k,h+1}\left(\overline{Q}_{h+1}^{k}-\underline{Q}_{h+1}^{k}\right)(s_{h+1}^k,a_{h+1}^k),
\end{align}
where
\[w_{k,h+1}=\sum_{i=n_h^k+1}^{N_h^k(s_h^k,a_h^k)}\alpha_i^{n_h^k}w_{\tau_h(s_h^k,a_h^k,i),h}.\]
We can prove that:
\begin{equation}
\sum_{k=1}^{K}w_{k,h}\beta_{n_h^k}\leqslant 10c\sqrt{SACw(2H)^3\iota}. 
\end{equation}
Also, we can prove that: $\{w_{k,h+1}\}_{k\in[K]}$ is a $\left(C,\left(1+\frac{1}{2H}\right)w\right)$-sequence. Therefore, after reversing this argument for $h+1,h+2,\ldots, 2H$, we obtain the following inequality:
\begin{equation*}
\begin{aligned}
\sum_{k=1}^{K}w_{k,h}\left(\overline{Q}_h^k-\underline{Q}_h^k\right)(s_h^k,a_h^k)&\leqslant \sum_{h'=0}^{2H-h}\left(2SAH\cdot\left(1+\frac{1}{2H}\right)^{h'}w+10c\sqrt{SAC\left(1+\frac{1}{2H}\right)^{h'}w(2H)^3\iota}\right)\\
&\leqslant 2H\left(2SAHew+10c\sqrt{SACew(2H)^3\iota}\right)\\
&=4ewSAH^2+60c\sqrt{SACewH^5\iota}.
\end{aligned}
\end{equation*}

\section{Proof of Lemma \ref{lemma-4}}
By using the conclusion of Lemma \ref{lemma-3}, we know that:
\begin{equation*}
\sum_{t=1}^{+\infty}\omega_t\left(\hat{Q}_t-\breve{Q}_t\right)(s_t,a_t)\leqslant \sum_{t=1}^{+\infty}\frac{2\omega_t\alpha_{n^t}^0}{1-\gamma}+\sum_{t=1}^{+\infty}2\omega_t\beta_{n^t}+\gamma\sum_{t=1}^{+\infty}\sum_{i=1}^{n^t}\omega_t\alpha_{n^t}^i\left(\hat{V}_{\tau(s,a,i)}-\breve{V}_{\tau(s,a,i)}\right)(s_{\tau(s,a,i)}).
\end{equation*}
Now, we analyze the three terms above one by one.
\begin{equation}
\label{eqn-lemma4-1} 
\sum_{t=1}^{+\infty}\frac{2\omega_t\alpha_{n^t}^0}{1-\gamma}\leqslant\sum_{t=1}^{+\infty}\mathbb{I}[n^t=0]\frac{2\omega}{1-\gamma}=\frac{2SA\omega}{1-\gamma},
\end{equation}
\begin{equation}
\label{eqn-lemma4-2}
\begin{aligned}
&~~\sum_{t=1}^{+\infty}\omega_t\beta_{n^t}=\sum_{s,a}\sum_{i=1}^{N(s,a)}\omega_{\tau(s,a,i)}\beta_i=\frac{c_3\sqrt{H}}{1-\gamma}\sum_{s,a}\sum_{i=1}^{N(s,a)}\omega_{\tau(s,a,i)}\sqrt{\frac{\iota(i)}{i}}\\
&\leqslant \frac{c_3\sqrt{H}}{1-\gamma}\sum_{s,a}\sum_{i=1}^{C_{s,a}/\omega}\omega\sqrt{\frac{\iota(C)}{i}}\leqslant \frac{2c_3\sqrt{H}}{1-\gamma}\sum_{s,a}\sqrt{C_{s,a}\omega\iota(C)} \leqslant \frac{2c_3}{1-\gamma}\sqrt{SAHC\omega \iota(C)}.
\end{aligned}
\end{equation}
Here, $C_{s,a}=\sum_{i=1}^{N(s,a)}\omega_{\tau(s,a,i)}$ is the partial sum of the sequence $\{\omega_t\}$, and therefore:
\[\sum_{s,a}C_{s,a}\leqslant C.\]
Equation \ref{eqn-lemma4-1} and Equation \ref{eqn-lemma4-2} stand for the first and second terms. Now we analyze the third term. 
\begin{equation}
\label{eqn-lemma4-3}
\begin{aligned}
&\gamma\sum_{t=1}^{+\infty}\sum_{i=1}^{n^t}\omega_t\alpha_{n^t}^i\left(\hat{V}_{\tau(s,a,i)}-\breve{V}_{\tau(s,a,i)}\right)(s_{\tau(s,a,i)})\\
=& \gamma\sum_{t=1}^{+\infty}\left(\hat{V}_t-\breve{V}_t\right)(s_{t+1})\sum_{i\geqslant n^t+1}\omega_{\tau(s_t,a_t,i)}\alpha_i^{n^t}\\
=& \gamma \sum_{t=2}^{+\infty}\omega_t'\left(\hat{V}_t-\breve{V}_t\right)(s_t)+\gamma\sum_{t=1}^{+\infty}\omega_{t+1}'\left(\hat{V}_t-\hat{V}_{t+1}\right)(s_{t+1})+\gamma\sum_{t=1}^{+\infty}\omega_{t+1}'\left(\breve{V}_{t+1}-\breve{V}_t\right)(s_{t+1}).
\end{aligned}
\end{equation}
Here:
\[\omega_{t+1}':=\sum_{i\geqslant n^t+1}\omega_{\tau(s_t,a_t,i)}\alpha_i^{n^t}\]
can be easily verified as a $(C,(1+1/H)\omega)$-sequence. By the update rule of Algorithm \ref{alg-2}, $\hat{Q}_t(s,a)$ is decreasing and $\breve{Q}_t(s,a)$ is increasing by $t$ for $\forall (s,a)\in\mS\times\mA$. Therefore, $\hat(V)_t(s)$ is decreasing and $\breve{V}_t(s)$ is increasing by $t$ for $\forall s\in\mS$. Then:
\[\gamma\sum_{t=1}^{+\infty}\omega_{t+1}'\left(\hat{V}_t-\hat{V}_{t+1}\right)(s_{t+1})\leqslant \gamma(1+1/H)\omega\sum_{s}\sum_{t=1}^{+\infty}\left(\hat{V}_t-\hat{V}_{t+1}\right)(s)\leqslant \frac{\gamma(1+1/H)\omega S}{1-\gamma},\]
and similarly:
\[\gamma\sum_{t=1}^{+\infty}\omega_{t+1}'\left(\breve{V}_{t+1}-\breve{V}_t\right)(s_{t+1})\leqslant \gamma(1+1/H)\omega\sum_{s}\sum_{t=1}^{+\infty}\left(\breve{V}_{t+1}-\breve{V}_t\right)(s)\leqslant \frac{\gamma(1+1/H)\omega S}{1-\gamma}.\]
Therefore, Equation (\ref{eqn-lemma4-3}) leads to:
\begin{equation*}
\begin{aligned}
\gamma\sum_{t=1}^{+\infty}\sum_{i=1}^{n^t}\omega_t\alpha_{n^t}^i\left(\hat{V}_{\tau(s,a,i)}-\breve{V}_{\tau(s,a,i)}\right)(s_{\tau(s,a,i)})\leqslant \gamma\sum_{t=2}^{+\infty}\omega_t'\left(\hat{Q}_t-\breve{Q}_t\right)(s_t,a_t)+\frac{2\gamma(1+1/H)\omega S}{1-\gamma}.
\end{aligned}    
\end{equation*}
After combining with Equation (\ref{eqn-lemma4-1}) and Equation (\ref{eqn-lemma4-2}), we obtain that:
\begin{equation*}
\begin{aligned}
&~~~~\sum_{t=1}^{+\infty}\omega_t\left(\hat{Q}_t-\breve{Q}_t\right)(s_t,a_t)\\
&\leqslant \frac{2SA\omega}{1-\gamma}+\frac{2c_3}{1-\gamma}\sqrt{SAHC\omega\iota(C)}+\frac{2\gamma(1+1/H)\omega S}{1-\gamma}+\gamma\sum_{t=2}^{+\infty}\omega_t'\left(\hat{Q}_t-\breve{Q}_t\right)(s_t,a_t)\\
&= \mathcal O\left(\frac{SA\omega+\sqrt{SAHC\omega\iota(C)}}{1-\gamma}\right) +\gamma\sum_{t=2}^{+\infty}\omega_t'\left(\hat{Q}_t-\breve{Q}_t\right)(s_t,a_t).
\end{aligned}    
\end{equation*}
We can repeat this unrolling argument for $H$ times, and get a $(C, (1+1/H)^H\omega\leqslant e\omega)$-sequence $\{\omega_t^{(H)}\}_{t\geqslant H+1}$. Then, we get the following result.
\begin{equation*}
\begin{aligned}
&~~~~\sum_{t=1}^{+\infty}\omega_t\left(\hat{Q}_t-\breve{Q}_t\right)(s_t,a_t)\\
&= \sum_{h=1}^{H}\gamma^h \mathcal O\left(\frac{SA\omega+\sqrt{SAHC\omega\iota(C)}}{1-\gamma}\right) +\gamma^H\sum_{t=H+1}^{+\infty}\omega_t^{(H)}\left(\hat{Q}_t-\breve{Q}_t\right)(s_t,a_t)\\
&\leqslant \frac{1}{1-\gamma}\cdot\mathcal O\left(\frac{SA\omega+\sqrt{SAHC\omega\iota(C)}}{1-\gamma}\right)+\frac{\gamma^H}{1-\gamma}\sum_{t=H+1}^{+\infty}\omega_t^{(H)}\\
&= \mathcal O\left(\frac{SA\omega+\sqrt{SAHC\omega\iota(C)}}{(1-\gamma)^2}\right)+\frac{\gamma^H C}{1-\gamma},
\end{aligned}    
\end{equation*}
which comes to our conclusion.

\section{Proof of Lemma \ref{lemma-5}}
Since $C^{(n)}=\sum_{t=1}^{+\infty}\omega_t^{(n)}$ and $\{\omega_t^{(n)}\}$ is a $(C^{(n)},1)$-sequence. According to Lemma \ref{lemma-4},
\begin{equation*}
\begin{aligned}
(2^{n-1}\gap^+_{\min})\cdot C^{(n)}&\leqslant \sum_{t=1}^{+\infty}\omega_t^{(n)}\left(\hat{Q}_t-\breve{Q}_t\right)(s_t,a_t)\leqslant \mathcal O\left(\frac{SA+\sqrt{SAHC^{(n)}\iota(C^{(n)})}}{(1-\gamma)^2}\right)+\frac{\gamma^H C^{(n)}}{1-\gamma}\\
&= \frac{\gap^+_{\min}}{2}C^{(n)}+\mathcal O\left(\frac{SA+\sqrt{SAHC^{(n)}\iota(C^{(n)})}}{(1-\gamma)^2}\right).
\end{aligned}    
\end{equation*}
Here, we use the fact that $\gamma^H=\frac{\gap^+_{\min}(1-\gamma)}{2}$. Denote $C^{(n)}=SAC'$, then:
\begin{equation*}
\begin{aligned}
&(2^{n-2}\gap^+_{\min})\cdot C^{(n)}\leqslant (2^{n-1}-\frac{1}{2})\gap^+_{\min}C^{(n)}\leqslant O\left(\frac{SA+\sqrt{SAHC^{(n)}\iota(C^{(n)})}}{(1-\gamma)^2}\right)\\
\Rightarrow & (2^{n-2}\gap^+_{\min})\cdot C'\leqslant O\left(\frac{1+\sqrt{HC'\iota(C^{(n)})}}{(1-\gamma)^2}\right)\leqslant O\left(\frac{1+\sqrt{HC'\log(SATC')}}{(1-\gamma)^2}\right).
\end{aligned}    
\end{equation*}
After solving the inequality above, we obtain that:
\[C'\leqslant \mathcal O\left(\frac{\log\left(\frac{SAT}{\gap^+_{\min}(1-\gamma)}\right)}{4^n (\gap^+_{\min})^2(1-\gamma)^4\log(1/\gamma)}\right).\]
Therefore,
\[C^{(n)}\leqslant \mathcal O\left(\frac{SA}{4^n (\gap^+_{\min})^2(1-\gamma)^4\log(1/\gamma)}\cdot \log\left(\frac{SAT}{\gap^+_{\min}(1-\gamma)}\right)\right),\]
which comes to our conclusion. 

\section{Proofs for 2-TBSG with Linear Function Expression}
In this section, we will give a theoretical proof on the Theorem \ref{thm-main-3}. First, we prove some common lemmas of both settings. 

\subsection{Concentration Properties}
The concentration property is important in controlling the fluctuations through the iterations. First, we introduce the following three lemmas proposed by \cite{jin2019provably}. 

\begin{Lemma}[Lemma B.3 of \cite{jin2019provably}]\label{lemma-conc-linear}
Under Assumption \ref{assumption-linear},  there exists an absolute constant $C$ that is independent of $c_{\beta}$, such that with probability at least $1-p$, the following event $\mathcal E_{\conc}$ holds: For centralized setting, 
\begin{equation*}
\begin{aligned}
\forall (k,h)\in[K]\times [2H]:~~&\left\|\sum_{i=1}^{k-1}\phi_h^i[\overline{V}_{h+1}^k(s_{h+1}^i)-\mathbb{P}_h \overline{V}_{h+1}^k(s_h^i,a_h^i)]\right\|_{(\Lambda_h^k)^{-1}}\leqslant C\cdot dH\sqrt{\theta},\\
&\left\|\sum_{i=1}^{k-1}\phi_h^i[\underline{V}_{h+1}^k(s_{h+1}^i)-\mathbb{P}_h \underline{V}_{h+1}^k(s_h^i,a_h^i)]\right\|_{(\Lambda_h^k)^{-1}}\leqslant C\cdot dH\sqrt{\theta},
\end{aligned}
\end{equation*}
For the independent setting, 
\begin{equation*}
\forall (k,h)\in[K]\times [2H]:~~\left\|\sum_{i=1}^{k-1}\phi_h^i[\overline{V}_{h+1}^k(s_{h+1}^i)-\mathbb{P}_h \overline{V}_{h+1}^k(s_h^i,a_h^i)]\right\|_{(\Lambda_h^k)^{-1}}\leqslant C\cdot dH\sqrt{\theta},\\
\end{equation*}

where $\theta=\log[2(1+c_{\beta})dT/p]$ and $\Lambda_h^k=\phi(s_h^k,a_h^k)$.
\end{Lemma}
By using Lemma \ref{lemma-conc-linear}, we can upper bound the difference between the optimal Q-function values and the Q-function values proposed by Algorithm \ref{alg-3}. 
\begin{Lemma}[Lemma B.4 of \cite{jin2019provably}]\label{lemma-conc2-linear}
For any fixed policy $\pi$, with the event $\mathcal E_{\conc}$ holds, we can conclude that $\forall (s,a,h,k)\in\mS\times\mA\times [2H]\times [K]$: For centralized setting:
\begin{equation*}
\begin{aligned}
\langle\phi(s,a), \overline{w}_h^k &\rangle -Q_h^{\pi}(s,a) = \mathbb{P}_h(\overline{V}_{h+1}^k-V_{h+1}^{\pi})(s,a)+\overline{\Delta}_h^k(s,a)\\
\langle\phi(s,a), \underline{w}_h^k &\rangle -Q_h^{\pi}(s,a) = \mathbb{P}_h(\underline{V}_{h+1}^k-V_{h+1}^{\pi})(s,a)+\underline{\Delta}_h^k(s,a).
\end{aligned}    
\end{equation*}
Here, $|\overline{\Delta}_h^k(s,a)|, |\underline{\Delta}_h^k(s,a)|\leqslant \beta\sqrt{\phi(s,a)^{\top}(\Lambda_h^k)^{-1}\phi(s,a)}$. If we make $\pi=(\pi^*,\mu^*)$, we have:
\begin{equation*}
\begin{aligned}
\langle\phi(s,a), \overline{w}_h^k &\rangle -Q_h^*(s,a) = \mathbb{P}_h(\overline{V}_{h+1}^k-V_{h+1}^*)(s,a)+\overline{\Delta}_h^k(s,a)\\
\langle\phi(s,a), \underline{w}_h^k &\rangle -Q_h^*(s,a) = \mathbb{P}_h(\underline{V}_{h+1}^k-V_{h+1}^*)(s,a)+\underline{\Delta}_h^k(s,a),
\end{aligned}    
\end{equation*}
and therefore,
\[\langle\phi(s,a), \overline{w}_h^k\rangle-\langle\phi(s,a), \underline{w}_h^k \rangle=\mathbb{P}_h(\overline{V}_{h+1}^k-\underline{V}_{h+1}^k)(s,a)+\overline{\Delta}_h^k(s,a)-\underline{\Delta}_h^k(s,a).\]
For the independent setting: 
\begin{equation*}
\langle\phi(s,a), w_h^k \rangle -Q_h^{\pi}(s,a) = \mathbb{P}_h(\overline{V}_{h+1}^k-V_{h+1}^{\pi})(s,a)+\overline{\Delta}_h^k(s,a)
\end{equation*}
where $|\overline{\Delta}_h^k(s,a)|\leqslant \beta\sqrt{\phi(s,a)^{\top}(\Lambda_h^k)^{-1}\phi(s,a)}$. If we make $\pi=(\mathrm{br}(\mu^k),\mu^k):=(\dagger, \mu^k)$ and $\pi=(\pi^*, \mu^*)$, we have:
\begin{equation*}
\begin{aligned}
\langle\phi(s,a), w_h^k \rangle -Q_h^{\dagger, \mu^k}(s,a) &= \mathbb{P}_h(\overline{V}_{h+1}^k-V_{h+1}^{\dagger, \mu^k})(s,a)+\overline{\Delta}_h^k(s,a),\\
\langle\phi(s,a), w_h^k \rangle -Q_h^*(s,a) &= \mathbb{P}_h(\overline{V}_{h+1}^k-V_{h+1}^*)(s,a)+\overline{\Delta}_h^k(s,a).
\end{aligned}    
\end{equation*}

\end{Lemma}

Finally, by using the method of induction, we obtain the following lemma, which shows that for centralized setting, $\overline{Q}, \underline{Q}$ are the upper and lower bounds of $Q^*$ respectively. For the independent setting, $\overline{Q}$ is the upper bound for every $Q^{\dagger, \mu^k}$. 
\begin{Lemma}[Lemma B.5 of \cite{jin2019provably}]\label{lemma-conc3-linear}
On the event $\mathcal E_{\conc}$ proposed in Lemma \ref{lemma-conc-linear}, we have: for the centralized setting, $\forall (s,a,h,k)\in\mS\times\mA\times[2H]\times [K]$: 
\[\overline{Q}_h^k(s,a)\geqslant Q_h^*(s,a)\geqslant \underline{Q}_h^k(s,a).\]
For the independent setting, $\forall (s,a,h,k)\in\mS\times\mA\times[2H]\times [K]$:
\[\overline{Q}_h^k(s,a)\geqslant Q_h^{\dagger, \mu^k}(s,a)\geqslant Q_h^*(s,a).\]
\end{Lemma}
We also need the following technical lemma. 
\begin{Lemma}[Lemma 6.5 of \cite{he2020logarithmic}]\label{lemma-conc4-linear}
When $\lambda=1$, for any subset $C=\{c_1,c_2,\ldots,c_k\}\subseteq [K]$ and any $h\in[2H]$, we have:
\[\sum_{i=1}^{k}(\phi_h^{c_i})^{\top}(\Lambda_h^{c_i})^{-1}\phi_h^{c_i}\leqslant 2d\log(1+k).\]
\end{Lemma}

\subsection{Classifying Positive Gaps into Intervals (Centralized version)}
Consider the term $\overline{Q}_h^k(s_h^k,a_h^k)-\underline{Q}_h^k(s_h^k,a_h^k)$. Under the event $\mathcal E_{\conc}$, when $h$ is odd:
\begin{equation*}
\begin{aligned}
\gap_h^+(s_h^k,a_h^k)&=V_h^*(s_h^k)-Q_h^*(s_h^k,a_h^k)\leqslant Q_h^*(s_h^k,a^*)-\underline{Q}_h^k(s_h^k,a_h^k)\\
&\leqslant \overline{Q}_h^k(s_h^k,a^*)-\underline{Q}_h^k(s_h^k,a_h^k)\leqslant \overline{Q}_h^k(s_h^k,a_h^k)-\underline{Q}_h^k(s_h^k,a_h^k)
\end{aligned}    
\end{equation*}
and when $h$ is even:
\begin{equation*}
\begin{aligned}
\gap_h^+(s_h^k,a_h^k)&=Q_h^*(s_h^k,a_h^k)-V_h^*(s_h^k)\leqslant Q_h^*(s_h^k,a_h^k)-Q_h^*(s_h^k,a^*)\\
&\leqslant \overline{Q}_h^k(s_h^k,a_h^k)-\underline{Q}_h^k(s_h^k,a^*)\leqslant \overline{Q}_h^k(s_h^k,a_h^k)-\underline{Q}_h^k(s_h^k,a_h^k).
\end{aligned}    
\end{equation*}
Here, we've applied Lemma \ref{lemma-conc3-linear}, and the way of choosing $a_h^k$. Since for $\forall (s,a,h,k)\in\mS\times \mA\times [2H]\times [K]$, the $\gap_h^+(s_h^k,a_h^k)\leqslant \overline{Q}_h^k(s_h^k,a_h^k)-\underline{Q}_h^k(s_h^k,a_h^k) \leqslant  2H$. So we can classify all these gaps into different intervals. According to the definition of $\gap^+_{\min}$, if one gap $\gap_h^+(s_h^k,a_h^k)$ belongs to $[0, \gap^+_{\min})$, then it must be 0. For the other gaps, we divide them into $N$ different intervals $[\gap^+_{\min}, 2\gap^+_{\min}), \ldots, [2^{N-1}\gap^+_{\min}, 2^{N}\gap^+_{\min})$. Here, $N=\lceil\log_{2}(2H/\gap^+_{\min})\rceil$. Then, we obtain the following conclusion:
\begin{equation}\label{eqn-split3}
\mathbb{E}[\mathrm{Regret}(K)]\leqslant \sum_{k=1}^{K}\mathbb{E}\left[\sum_{h=1}^{2H}\gap^{+}_h(s_h^k,a_h^k)\right]\leqslant \sum_{h=1}^{2H}\mathbb{E}\left[\sum_{n=1}^{N}2^n\gap^+_{\min}\cdot\mathcal T_h^{(n)} \right].   
\end{equation}
where $\mathcal T_k^{(n)}=\sum_{k=1}^{K}\mathbb I\left[\gap^{+}_h(s_h^k,a_h^k)\in[2^{n-1}\gap^+_{\min}, 2^n\gap^+_{\min})\right]$, which is the number of positive gaps at the $h$-th step that belongs to the interval $[2^{n-1}\gap^+_{\min}, 2^n\gap^+_{\min})$ during the first $K$ episodes.

\subsection{Classifying Positive Gaps into Intervals (Independent version)}
For the independent setting, notice that:
\[\mathbb{E}[\mathrm{Regret}_{\mu}(K)]=\sum_{k=1}^{K}\mathbb{E}\left[V_1^{\dagger, \mu^k}(s_1^k)-V_1^{\pi^k,\mu^k}(s_1^k)\right]=\sum_{k=1}^{K}\mathbb{E}\left[\sum_{\substack{1\leqslant h \leqslant 2H\\h \text{~odd}}}\gap_{h}^{\mu}(s_h^k, a_h^k)|\pi^k,\mu^k\right].\]

Consider the term $\overline{Q}_h^k(s_h^k,a_h^k)-Q_h^*(s_h^k,a_h^k)$. Under the event $\mathcal E_{\conc}$, when $h$ is odd:

\begin{equation*}
\begin{aligned}
\gap_h^{\mu}(s_h^k,a_h^k)&=V_h^{\dagger, \mu^k}(s_h^k)-Q_h^{\dagger, \mu^k}(s_h^k,a_h^k)= Q_h^{\dagger, \mu^k}(s_h^k,\hat{a})-Q_h^{\dagger, \mu^k}(s_h^k,a_h^k)\\
&\leqslant \overline{Q}_h^k(s_h^k,\hat{a})-Q_h^*(s_h^k,a_h^k)\leqslant \overline{Q}_h^k(s_h^k,a_h^k)-Q_h^*(s_h^k,a_h^k).
\end{aligned}    
\end{equation*}
Here, we've applied Lemma \ref{lemma-conc3-linear}, and the way of choosing $a_h^k$. Since in the regret decomposition above, we only need the gap with odd steps, we have:
\begin{equation*}
\mathbb{E}[\mathrm{Regret}(K)]\leqslant \sum_{k=1}^{K}\mathbb{E}\left[\sum_{h=1}^{2H}\left(\overline{Q}_h^k(s_h^k,a_h^k)-Q_h^*(s_h^k,a_h^k)\right)\right].
\end{equation*}

Since for $\forall (s,a,h,k)\in\mS\times \mA\times [2H]\times [K]$, the $\gap_h^{\mu}(s_h^k,a_h^k)\leqslant \overline{Q}_h^k(s_h^k,a_h^k)-Q_h^*(s_h^k,a_h^k) \leqslant  2H$ when $h$ is an odd number. So we can classify all these non-zero gaps into $N$ different intervals $[\gap^+_{\min}, 2\gap^+_{\min}), \ldots, [2^{N-1}\gap^+_{\min}, 2^{N}\gap^+_{\min})$. Here, $N=\lceil\log_{2}(2H/\gap^+_{\min})\rceil$. Then, we obtain the following conclusion:
\begin{equation}\label{eqn-split4}
\mathbb{E}[\mathrm{Regret}(K)]\leqslant \sum_{k=1}^{K}\mathbb{E}\left[\sum_{\substack{1\leqslant h \leqslant 2H\\h \text{~odd}}}\gap_{h}^{\mu}(s_h^k, a_h^k)\right]\leqslant \sum_{h=1}^{2H}\mathbb{E}\left[\sum_{n=1}^{N}2^n\gap^+_{\min}\cdot\mathcal T_h^{(n)} \right].   
\end{equation}
where $\mathcal T_k^{(n)}=\sum_{k=1}^{K}\mathbb I\left[\gap^{\mu}_h(s_h^k,a_h^k)\in[2^{n-1}\gap^+_{\min}, 2^n\gap^+_{\min}), h~\text{is odd}\right]$, which is the number of positive gaps at the $h$-th step that belongs to the interval $[2^{n-1}\gap^+_{\min}, 2^n\gap^+_{\min})$ during the first $K$ episodes. In the next two sections, we will upper bound the counting number $\mathcal T_h^{(n)}$, which is the final step of our proof. By using Equation (\ref{eqn-split3}, \ref{eqn-split4}), we can upper bound the expected total regret in both centralized and independent settings.

\subsection{Upper Bounding the Counting Number (Centralized version)}
For a fixed $h\in[2H]$ and $n\leqslant \lceil \log_2(2H/\gap^+_{\min})\rceil$, we will upper bound the $\mathcal T_k^{(n)}$. Under $\mathcal E_{\conc}$, denote:
\[\{k\in[K]~:~\gap_h^+(s_h^k,a_h^k)\in[2^{n-1}\gap^+_{\min}, 2^n\gap^+_{\min})\}=\{k_1, k_2,\ldots,k_t\}:=\mathcal D,\]
where $t=\mathcal T_h^{(n)}$, then consider the sum $\sum_{i=1}^{t}[\overline{Q}_h^k(s_h^k,a_h^k)-\underline{Q}_h^k(s_h^k,a_h^k)]$. 
On one hand, this sum has a lower bound:
\begin{equation}\label{eqn-linear-lower}
\sum_{k\in\mathcal D}[\overline{Q}_h^k(s_h^k,a_h^k)-\underline{Q}_h^k(s_h^k,a_h^k)] \geqslant \sum_{k\in\mathcal D} \gap_h^+(s_h^k,a_h^k)\geqslant 2^{n-1}\gap^+_{\min}\cdot \mathcal T_h^{(n)}.
\end{equation}
On the other hand, we can also establish an upper bound. By using Lemma \ref{lemma-conc2-linear}:
\begin{align}\label{eqn-linear-upper}
&~~~\sum_{k\in\mathcal D}[\overline{Q}_h^k(s_h^k,a_h^k)-\underline{Q}_h^k(s_h^k,a_h^k)] \leqslant  \sum_{k\in\mathcal D}\left[2\beta\sqrt{\phi(s_h^k,a_h^k)^{\top}(\Lambda_h^k)^{-1}\phi(s_h^k,a_h^k)}+\phi(s_h^k,a_h^k)^{\top}(\overline{w}_h^k-\underline{w}_h^k)\right]\notag\\
&\leqslant 2\beta\sum_{k\in\mathcal D}\|\phi(s_h^k,a_h^k)\|_{(\Lambda_h^k)^{-1}}+\sum_{k\in\mathcal D}\left[\phi(s_h^k,a_h^k)^{\top}\overline{w}_h^k-\phi(s_h^k,a_h^k)^{\top}\underline{w}_h^k\right]\notag\\
&\leqslant 2\beta\sum_{k\in\mathcal D}\|\phi(s_h^k,a_h^k)\|_{(\Lambda_h^k)^{-1}}+\sum_{k\in\mathcal D}\left[\mathbb{P}_h(\overline{V}_{h+1}^k-\underline{V}_{h+1}^k)(s_h^k,a_h^k)+2\beta\|\phi(s_h^k,a_h^k)\|_{(\Lambda_h^k)^{-1}}\right]\notag\\
&= 4\beta\sum_{k\in\mathcal D}\|\phi(s_h^k,a_h^k)\|_{(\Lambda_h^k)^{-1}} + \sum_{k\in\mathcal D}[\overline{V}_{h+1}^k(s_{h+1}^k)-\underline{V}_{h+1}^k(s_{h+1}^k)]+\sum_{k\in\mathcal D}\varepsilon_h^k\notag\\
&= 4\beta\sum_{k\in\mathcal D}\|\phi(s_h^k,a_h^k)\|_{(\Lambda_h^k)^{-1}} + \sum_{k\in\mathcal D}[\overline{Q}_{h+1}^k(s_{h+1}^k,a_{h+1}^k)-\underline{Q}_{h+1}^k(s_{h+1}^k,a_{h+1}^k)]+\sum_{k\in\mathcal D}\varepsilon_h^k.
\end{align}
Taking summation over $h'=h,h+1,\ldots,2H$, we have:
\begin{equation}\label{eqn-linear-2}
\sum_{k\in\mathcal D}[\overline{Q}_h^k(s_h^k,a_h^k)-\underline{Q}_h^k(s_h^k,a_h^k)]\leqslant 4\beta\sum_{h'=h}^{2H}\sum_{k\in\mathcal D}\|\phi(s_{h'}^k,a_{h'}^k)\|_{(\Lambda_{h'}^k)^{-1}}+\sum_{h'=h}^{2H}\sum_{k\in\mathcal D}\varepsilon_{h'}^k.
\end{equation}
Here, $\varepsilon_h^k=\mathbb{P}_h(\overline{V}_{h+1}^k-\underline{V}_{h+1}^k)(s_h^k,a_h^k)-(\overline{V}_{h+1}^k-\underline{V}_{h+1}^k)(s_{h+1}^k)$, which forms a martingale difference sequence. For each $k\in[K]$, with probability at least $1-p$, it holds that:
\[\sum_{i=1}^{k}\sum_{h'=h}^{2H}\varepsilon_{h'}^i\leqslant \sqrt{8kH^2\log(2/p)}.\]
After taking a union bound for all $k\in[K]$, we know that, with probability at least $1-Kp$, it holds that:
\begin{equation}\label{eqn-linear-1}
\sum_{k\in\mathcal D}\sum_{h'=h}^{2H}\varepsilon_h^i\leqslant \sqrt{8\mathcal T_{h'}^{(n)}H^2 \log(2/p)}.    
\end{equation}
According to Lemma \ref{lemma-conc4-linear} and Cauchy Inequality, we have:
\begin{equation}\label{eqn-linear-3}
\begin{aligned}
\sum_{k\in\mathcal D}\|\phi(s_h^k,a_h^k)\|_{(\Lambda_h^k)^{-1}} &\leqslant \sqrt{\mathcal T_h^{(n)}}\cdot\sqrt{\sum_{k\in\mathcal D}\|\phi(s_h^k,a_h^k)\|_{(\Lambda_h^k)^{-1}}^2}= \sqrt{\mathcal T_h^{(n)}}\cdot\sqrt{\sum_{k\in\mathcal D}(\phi_h^k)^{\top}(\Lambda_h^k)^{-1}\phi_h^k}\\
&\leqslant \sqrt{2d\mathcal T_h^{(n)}\cdot\log(1+\mathcal T_h^{(n)})}.
\end{aligned}
\end{equation}
Combine the martingale concentration with $\mathcal E_{\conc}$, by using Equation (\ref{eqn-linear-2}, \ref{eqn-linear-1}, \ref{eqn-linear-3}), we have: with probability at least $1-(K+1)p$,
\begin{equation}\label{eqn-linear-4}
\sum_{k\in\mathcal D}[\overline{Q}_h^k(s_h^k,a_h^k)-\underline{Q}_h^k(s_h^k,a_h^k)] \leqslant 4\beta\sqrt{2d\mathcal T_h^{(n)}\cdot\log(1+\mathcal T_h^{(n)})}+\sqrt{8\mathcal T_{h'}^{(n)}H^2 \log(2/p)}.
\end{equation}
Finally, by Equation (\ref{eqn-linear-lower}) and Equation (\ref{eqn-linear-4}), we know that with probability at least $1-(K+1)p$:
\begin{equation*}
\begin{aligned}
2^{n-1}\gap^+_{\min}\cdot \mathcal T_h^{(n)}&\leqslant 2H\cdot 4\beta\sqrt{2d\mathcal T_h^{(n)}\cdot\log(1+\mathcal T_h^{(n)})}+\sqrt{8\mathcal T_{h'}^{(n)}H^2 \log(2/p)}\\
&= 8c_{\beta}dH^2\sqrt{\log(4dKH/p)}\cdot\sqrt{2d\mathcal T_h^{(n)}\cdot\log(1+\mathcal T_h^{(n)})}+\sqrt{8\mathcal T_{h'}^{(n)}H^2 \log(2/p)}.
\end{aligned}    
\end{equation*}
Therefore, we conclude that there exists an absolute constant $C$ such that:
\begin{equation}
\label{eqn-linear-uppergap}    
\mathcal T_h^{(n)}\leqslant\frac{Cd^3H^4\log(4dKH/p)}{4^n(\gap^+_{\min})^2}\cdot\log\left(\frac{Cd^3H^4\log(4dKH/p)}{4^n(\gap^+_{\min})^2}\right).
\end{equation}

\subsection{Upper Bounding the Counting Number (Independent version)}
For a fixed odd number $h\in[2H]$ and $n\leqslant \lceil \log_2(2H/\gap^+_{\min})\rceil$, we will upper bound the $\mathcal T_k^{(n)}$. Under $\mathcal E_{\conc}$, again we denote:
\[\{k\in[K]~:~\gap_h^+(s_h^k,a_h^k)\in[2^{n-1}\gap^+_{\min}, 2^n\gap^+_{\min})\}=\{k_1, k_2,\ldots,k_t\}:=\mathcal D,\]
where $t=\mathcal T_h^{(n)}$, then consider the sum $\sum_{i=1}^{t}[\overline{Q}_h^k(s_h^k,a_h^k)-\underline{Q}_h^k(s_h^k,a_h^k)]$. 
On one hand, this sum has a lower bound:
\begin{equation}\label{eqn-linear-lower-2}
\sum_{k\in\mathcal D}[\overline{Q}_h^k(s_h^k,a_h^k)-Q_h^*(s_h^k,a_h^k)] \geqslant \sum_{k\in\mathcal D} \gap_h^+(s_h^k,a_h^k)\geqslant 2^{n-1}\gap^+_{\min}\cdot \mathcal T_h^{(n)}.
\end{equation}
On the other hand, we can also establish an upper bound. By using Lemma \ref{lemma-conc2-linear}:
\begin{align}\label{eqn-linear-upper-2}
&~~~\sum_{k\in\mathcal D}[\overline{Q}_h^k(s_h^k,a_h^k)-Q_h^*(s_h^k,a_h^k)] \leqslant  \sum_{k\in\mathcal D}\left[\beta\sqrt{\phi(s_h^k,a_h^k)^{\top}(\Lambda_h^k)^{-1}\phi(s_h^k,a_h^k)}+\langle\phi(s_h^k,a_h^k), w_h^k\rangle - Q^*_h(s_h^k,a_h^k)\right]\notag\\
&\leqslant 2\beta\sum_{k\in\mathcal D}\|\phi(s_h^k,a_h^k)\|_{(\Lambda_h^k)^{-1}}+\sum_{k\in\mathcal D}\mathbb{P}_h(\overline{V}_{h+1}^k-V^*_{h+1})(s_h^k,a_h^k)\notag\\
&= 2\beta\sum_{k\in\mathcal D}\|\phi(s_h^k,a_h^k)\|_{(\Lambda_h^k)^{-1}} + \sum_{k\in\mathcal D}[\overline{V}_{h+1}^k(s_{h+1}^k)-V_{h+1}^*(s_{h+1}^k)]+\sum_{k\in\mathcal D}\varepsilon_h^k\notag\\
&\leqslant 2\beta\sum_{k\in\mathcal D}\|\phi(s_h^k,a_h^k)\|_{(\Lambda_h^k)^{-1}} + \sum_{k\in\mathcal D}[\overline{Q}_{h+1}^k(s_{h+1}^k,a_{h+1}^k)-Q_{h+1}^*(s_{h+1}^k,a_{h+1}^k)]+\sum_{k\in\mathcal D}\varepsilon_h^k.
\end{align}
The final step holds since 
\[V_{h+1}^*(s_{h+1}^k)=\max_{a}Q_{h+1}^*(s_{h+1}^k, a)\geqslant Q_{h+1}^*(s_{h+1}^k, a_{h+1}^k).\]
After taking summation over $h'=h,h+1,\ldots,2H$, we have:
\begin{equation}\label{eqn-linear-2-2}
\sum_{k\in\mathcal D}[\overline{Q}_h^k(s_h^k,a_h^k)-Q_h^*(s_h^k,a_h^k)]\leqslant 2\beta\sum_{h'=h}^{2H}\sum_{k\in\mathcal D}\|\phi(s_{h'}^k,a_{h'}^k)\|_{(\Lambda_{h'}^k)^{-1}}+\sum_{h'=h}^{2H}\sum_{k\in\mathcal D}\varepsilon_{h'}^k.
\end{equation}
Here, $\varepsilon_h^k=\mathbb{P}_h(\overline{V}_{h+1}^k-V_{h+1}^*)(s_h^k,a_h^k)-(\overline{V}_{h+1}^k-V_{h+1}^*)(s_{h+1}^k)$, which forms a martingale difference sequence. For each $k\in[K]$, with probability at least $1-p$, it holds that:
\[\sum_{i=1}^{k}\sum_{h'=h}^{2H}\varepsilon_{h'}^i\leqslant \sqrt{8kH^2\log(2/p)}.\]
The following is exact the same as the centralized version. We take a union bound for all $k\in[K]$, we know that, with probability at least $1-Kp$, it holds that:
\begin{equation}\label{eqn-linear-1-2}
\sum_{k\in\mathcal D}\sum_{h'=h}^{2H}\varepsilon_h^i\leqslant \sqrt{8\mathcal T_{h'}^{(n)}H^2 \log(2/p)}.    
\end{equation}

Combine the martingale concentration with $\mathcal E_{\conc}$, by using Equations \eqref{eqn-linear-2-2}, \eqref{eqn-linear-1-2}, \eqref{eqn-linear-3},  with probability at least $1-(K+1)p$, we have 
\begin{equation}\label{eqn-linear-4-2}
\sum_{k\in\mathcal D}[\overline{Q}_h^k(s_h^k,a_h^k)-Q_h^*(s_h^k,a_h^k)] \leqslant 2\beta\sqrt{2d\mathcal T_h^{(n)}\cdot\log(1+\mathcal T_h^{(n)})}+\sqrt{8\mathcal T_{h'}^{(n)}H^2 \log(2/p)}.
\end{equation}
Finally, by Equation (\ref{eqn-linear-4-2}), we know that with probability at least $1-(K+1)p$:
\begin{equation*}
\begin{aligned}
2^{n-1}\gap^+_{\min}\cdot \mathcal T_h^{(n)}&\leqslant 2H\cdot 2\beta\sqrt{2d\mathcal T_h^{(n)}\cdot\log(1+\mathcal T_h^{(n)})}+\sqrt{8\mathcal T_{h'}^{(n)}H^2 \log(2/p)}\\
&= 4c_{\beta}dH^2\sqrt{\log(4dKH/p)}\cdot\sqrt{2d\mathcal T_h^{(n)}\cdot\log(1+\mathcal T_h^{(n)})}+\sqrt{8\mathcal T_{h'}^{(n)}H^2 \log(2/p)}.
\end{aligned}    
\end{equation*}
Therefore, we conclude that there exists an absolute constant $C$ such that:
\begin{equation}
\label{eqn-linear-uppergap-2}    
\mathcal T_h^{(n)}\leqslant\frac{Cd^3H^4\log(4dKH/p)}{4^n(\gap^+_{\min})^2}\cdot\log\left(\frac{Cd^3H^4\log(4dKH/p)}{4^n(\gap^+_{\min})^2}\right).
\end{equation}

\subsection{Final Step of the Proof}
Here, we come to final step of our proof. For the centralized setting, we combine Equation (\ref{eqn-split3}) with the upper bound of counting numbers (Equation (\ref{eqn-linear-uppergap})), we have:
\begin{equation*}
\begin{aligned}
&~~~\mathbb{E}[\mathrm{Regret}(K)]\leqslant \sum_{h=1}^{2H}\mathbb{E}\left[\sum_{n=1}^{N}2^n\gap^+_{\min}\cdot\mathcal T_h^{(n)} \right] \\
&\leqslant (K+1)p\cdot 4H^2K +\sum_{n=1}^{N}2^n\gap^+_{\min}\cdot\frac{Cd^3H^5\log(4dKH/p)}{4^n(\gap^+_{\min})^2}\cdot\log\left(\frac{Cd^3H^4\log(4dKH/p)}{4^n(\gap^+_{\min})^2}\right)\\
&\leqslant 4H^2K(K+1)p +\frac{Cd^3H^5\log(4dKH/p)}{\gap^+_{\min}}\cdot\log\left(\frac{Cd^3H^4\log(4dKH/p)}{(\gap^+_{\min})^2}\right).
\end{aligned}    
\end{equation*}
Let $p=\frac{1}{4H^2K(K+1)}$, we can rewrite the inequality above as:
\[\mathbb{E}[\mathrm{Regret}(K)]\leqslant 1+\frac{Cd^3H^5\log(16dK^2(K+1)H^3)}{\gap^+_{\min}}\iota,\]
where $\iota=\log\left(\frac{Cd^3H^4\log(4dKH/p)}{(\gap^+_{\min})^2}\right)$ is a logarithmic term. Till now, our main theorem for centralized setting is proved. The independent version is very similar. After combining Equation (\ref{eqn-split4}) with the counting number upper bound (Equation (\ref{eqn-linear-uppergap-2})), we can get exactly the same result: for any sequence of policies $\mu:=\{\mu^k\}_{k\in[K]}$, we have:
\[\mathbb{E}[\mathrm{Regret}_{\mu}(K)]\leqslant 1+\frac{Cd^3H^5\log(16dK^2(K+1)H^3)}{\gap^+_{\min}}\iota,\]
where $\iota=\log\left(\frac{Cd^3H^4\log(4dKH/p)}{(\gap^+_{\min})^2}\right)$ is a logarithmic term.

\newpage
\section{The Pseudo-code of LSVI-2TBSG}
In this section, we introduce the formal pseudo-code of LSVI-2TBSG algorithm in both centralized setting and independent setting. 
\begin{algorithm}[!ht]
\caption{Optimistic Nash Q-learning on two-player Turn-based Stochastic Games with Linear Function Expression (Centralized)}
\hspace*{0.02in}{\bf Initialize:}
Let $\overline{Q}_h^1(s,a)\leftarrow 2H, \underline{Q}_h^1(s,a)\leftarrow 0$, and $ \overline{V}_h^1(s,a)\leftarrow 2H, \underline{V}_h^1(s,a)\leftarrow 0$ for all $(s,a,h)\in\mS\times \mA\times [2H]$. $\lambda\leftarrow 1, \beta\leftarrow c_{\beta}\cdot dH\sqrt{\iota}$ where $c_{\beta}=160$ is an absolute constant and $\iota=\log(2dT/p)=\log(4dKH/p)$.\\
\begin{algorithmic}[1]
\FOR{episode $k\in[K]$}
\STATE{Observe the initial state $s_1^k$.}
\FOR{step $h = 2H, 2H-1, \ldots, 1$}
\STATE $\Lambda_h^k=\sum_{i=1}^{k-1}\phi(s_h^i,a_h^i)\phi(s_h^i,a_h^i)^{\top}+\lambda I$
\STATE $\overline{w}_h^k=(\Lambda_h^k)^{-1}\sum_{i=1}^{k-1}\phi(s_h^i,a_h^i)\left[r_h(s_h^i,a_h^i)+\overline{V}_{h+1}^k(s_{h+1}^i)\right]$
\STATE $\underline{w}_h^k=(\Lambda_h^k)^{-1}\sum_{i=1}^{k-1}\phi(s_h^i,a_h^i)\left[r_h(s_h^i,a_h^i)+\underline{V}_{h+1}^k(s_{h+1}^i)\right]$
\STATE $\overline Q_h^k(s,a)\leftarrow \min\left(2H, \phi(s,a)^{\top}\overline{w}_h^k + \beta \sqrt{\phi(s,a)^{\top}(\Lambda_h^k)^{-1}\phi(s,a)}\right)$
\STATE $\underline Q_h^k(s,a)\leftarrow \max\left(0, \phi(s,a)^{\top}\underline{w}_h^k - \beta \sqrt{\phi(s,a)^{\top}(\Lambda_h^k)^{-1}\phi(s,a)}\right)$
\ENDFOR
\FOR{step $h=1,2,\ldots,2H$}
\STATE If $h$ is an odd number, take action $a_h^k\leftarrow \arg\max_a \overline{Q}_h^k(s_h^k, a)$, otherwise, take action $a_h^k\leftarrow \arg\min_a \underline{Q}_h^k(s_h^k, a)$. Then, we receive the next state $s_{h+1}^k$.
\STATE $\overline{V}_h^k(s_h^k)\leftarrow\overline{Q}_h^k(s_h^k,a_h^k), ~\underline{V}_h^k(s_h^k)\leftarrow \underline{Q}_h^k(s_h^k,a_h^k)$.
\ENDFOR
\ENDFOR
\end{algorithmic}
\label{alg-3}
\end{algorithm}

\begin{algorithm}[!ht]
\caption{Optimistic Nash Q-learning on two-player Turn-based Stochastic Games with Linear Function Expression (Independent)}
\hspace*{0.02in}{\bf Initialize:}
Let $\overline{Q}_h^1(s,a)\leftarrow 2H$, and $ \overline{V}_h^1(s,a)\leftarrow 2H$ for all $(s,a,h)\in\mS\times \mA\times [2H]$. $\lambda\leftarrow 1, \beta\leftarrow c_{\beta}\cdot dH\sqrt{\iota}$ where $c_{\beta}=160$ is an absolute constant and $\iota=\log(2dT/p)=\log(4dKH/p)$.\\
\begin{algorithmic}[1]
\FOR{episode $k\in[K]$}
\STATE{Observe the initial state $s_1^k$.}
\FOR{step $h = 2H, 2H-1, \ldots, 1$}
\STATE $\Lambda_h^k=\sum_{i=1}^{k-1}\phi(s_h^i,a_h^i)\phi(s_h^i,a_h^i)^{\top}+\lambda I$
\STATE $w_h^k=(\Lambda_h^k)^{-1}\sum_{i=1}^{k-1}\phi(s_h^i,a_h^i)\left[r_h(s_h^i,a_h^i)+\overline{V}_{h+1}^k(s_{h+1}^i)\right]$
\STATE $\overline Q_h^k(s,a)\leftarrow \min\left(2H, \phi(s,a)^{\top}w_h^k + \beta \sqrt{\phi(s,a)^{\top}(\Lambda_h^k)^{-1}\phi(s,a)}\right)$
\ENDFOR
\FOR{step $h=1,2,\ldots,2H$}
\STATE If $h$ is an odd number, take action $a_h^k\leftarrow \arg\max_a \overline{Q}_h^k(s_h^k, a)$, otherwise, let the min-player choose his action $a_h^k$. Then, we receive the next state $s_{h+1}^k$.
\STATE $\overline{V}_h^k(s_h^k)\leftarrow\overline{Q}_h^k(s_h^k,a_h^k)$.
\ENDFOR
\ENDFOR
\end{algorithmic}
\label{alg-4}
\end{algorithm}

\end{document}